\documentclass[twoside]{article}

\usepackage[preprint]{aistats2026}
\usepackage[round]{natbib}

\usepackage{xcolor}
\usepackage{hyperref}
\usepackage{url}
\usepackage{amsthm}
\usepackage{amssymb}
\usepackage{graphicx}
\usepackage{subfigure}
\usepackage{float}
\usepackage{enumitem}
\usepackage{amsmath}
\usepackage{dsfont}
\usepackage{bm}
\usepackage{wrapfig}
\usepackage[capitalise]{cleveref}
\usepackage{algorithm}
\usepackage{algorithmic}

\newtheorem{theorem}{Theorem}
\newtheorem{lemma}{Lemma}
\newtheorem{definition}{Definition}

\newtheorem{assumption}{Assumption}
\newtheorem{corollary}{Corollary}

\begin{document}

%

%

\twocolumn[

\aistatstitle{Spectral Alignment as Predictor of Loss Explosion in Neural Network Training}

\aistatsauthor{ Haiquan Qiu$^1$, You Wu$^1$, Yingjie Tan$^1$, Yaqing Wang$^2$, Quanming Yao$^1$}

\aistatsaddress{ $^1$Tsinghua University \\ $^2$Beijing Institute of Mathematical Sciences and Applications \\
\texttt{\{qhq22, wu-you21, tanyj23\}@mails.tsinghua.edu.cn,} \\
\texttt{wangyaqing@bimsa.cn, qyaoaa@tsinghua.edu.cn}
} 
]

\begin{abstract}
Loss explosions in training deep neural networks can nullify multi-million dollar training runs. Conventional monitoring metrics like weight and gradient norms are often lagging and ambiguous predictors, as their values vary dramatically across different models and even between layers of the same model, making it difficult to establish a unified standard for detecting impending failure. We introduce Spectral Alignment (SA), a novel, theoretically-grounded metric that monitors the distributional alignment between layer inputs and the principal singular vectors of weight matrices. We show that a collapse in the sign diversity of this alignment is a powerful early predictor of representational collapse and training divergence. Empirical results on language models demonstrate that monitoring the SA distribution provides a significantly earlier and clearer warning of loss explosions than traditional scalar metrics. SA's low computational overhead makes it a practical tool for safeguarding model training.
\end{abstract}

\section{Introduction}

\begin{figure}[t]
    \centering
    \includegraphics[width=0.5\textwidth]{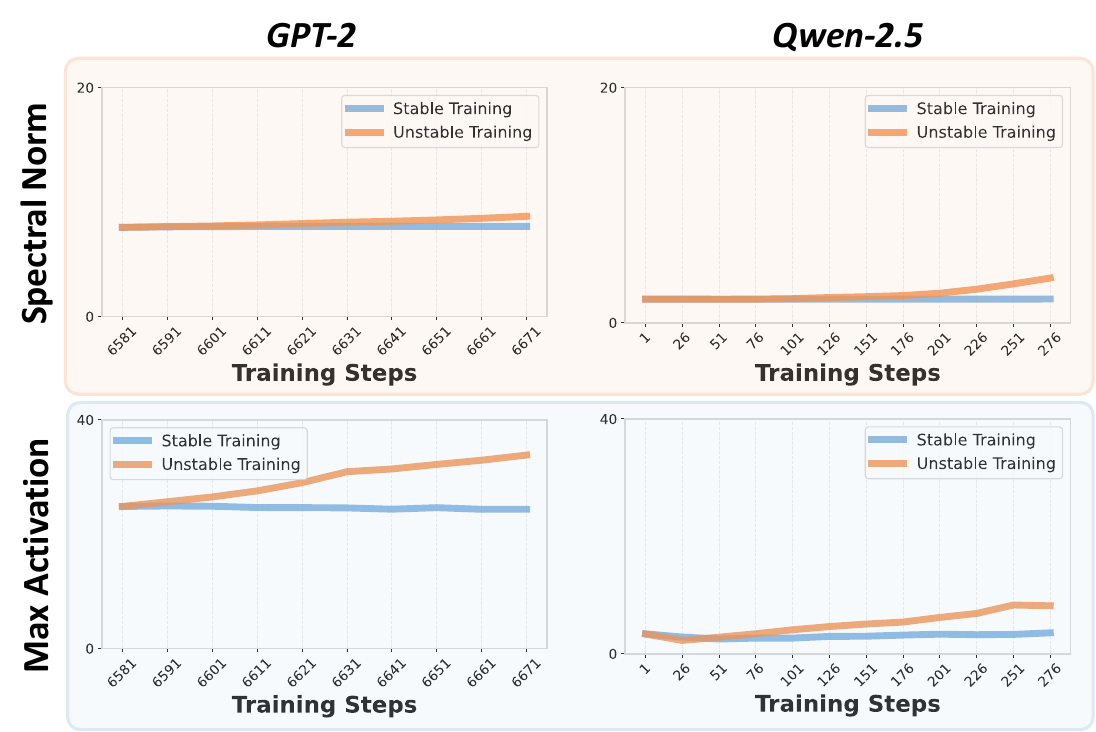}
    \vspace{-10px}
    \caption{Some conventional metrics, such as weight norm and maximal activation. This figure shows two models investigated in our paper, which shows the inconsistent patterns for conventional metrics across different models. This makes it is difficult to establish a universal threshold for loss explosion prediction.}
    \label{fig:intro_fig}
\end{figure}

The pursuit of ever-more capable artificial intelligence has led to the era of large-scale models, where training a single model can consume millions of dollars in computational resources and weeks of time \citep{brown2020language,touvron2023llama}. In this high-stakes environment, loss explosions are catastrophic failures. A sudden divergence, often occurring after days of seemingly successful training, can nullify the entire investment. These instabilities frequently arise from the interplay between modern training techniques and hyperparameter choices. For instance, the widespread adoption of low-precision formats like bfloat16 to accelerate training and reduce memory usage introduces a trade-off with numerical stability; the limited dynamic range can lead to underflow or overflow, causing errors that accumulate and trigger sudden divergence \citep{micikevicius2017mixed}. Similarly, aggressive hyperparameter settings, such as a high learning rate, can cause training to overshoot stable regions of the loss landscape and spiral out of control~\citep{wortsman2023small}. The ability to predict such failures before they manifest is one of the most pressing practical challenges in modern deep learning. Yet, finding reliable methods for monitoring training stability remains an open issue.

Traditional methods for monitoring loss explosions, such as tracking the norm of gradients or weights\citep{takase2025spike,fishman2024scaling,zhai2023sreparam}, stable rank~\citep{Sanyal2020Stable}, or the magnitude of activations~\citep{rybakov2024methods,wortsman2023small}, offer scalar summaries of the training process. However, these indicators have significant limitations. A primary issue is that their raw values are not directly comparable; they can vary dramatically not only between different models but also across various layers within the same model (see \cref{fig:intro_fig}).
This variance makes it difficult to establish a universal threshold that reliably signals an impending failure. Moreover, the growth rates of these metrics are often ambiguous. A rising norm, for instance, could indicate healthy feature learning rather than nascent failure. Distinguishing between a normal rate of increase and a precursor to a training crash is a complex, often model-specific task. Consequently, these metrics are often lagging predictors, revealing a problem only after it has become critical, rather than serving as a robust, predictive tool.

Our approach is built on the theoretical insight that feature learning occurs as layer inputs align with the principal singular vectors of their corresponding weight matrices \citep{yang2022tensor,yang2023spectral,imani2021representation}. While this alignment drives feature learning, we argue that its \emph{distributional character} across a batch of inputs is the key to healthy training. A healthy layer should respond discriminatively, with its principal singular vector aligning positively with some inputs and negatively with others. We posit that a collapse in this sign diversity—where nearly all inputs are pushed in the same direction—is a powerful early-warning signal. This loss of diversity indicates the onset of representational collapse via a dangerous amplification loop, a precursor to the loss explosions.

\begin{figure*}
\subfigure[Spectral alignment]{\includegraphics[width=0.3\textwidth]{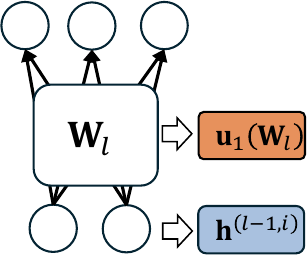}\label{fig:sa_concept}}
\subfigure[Stable Loss]{\includegraphics[width=0.3\textwidth]{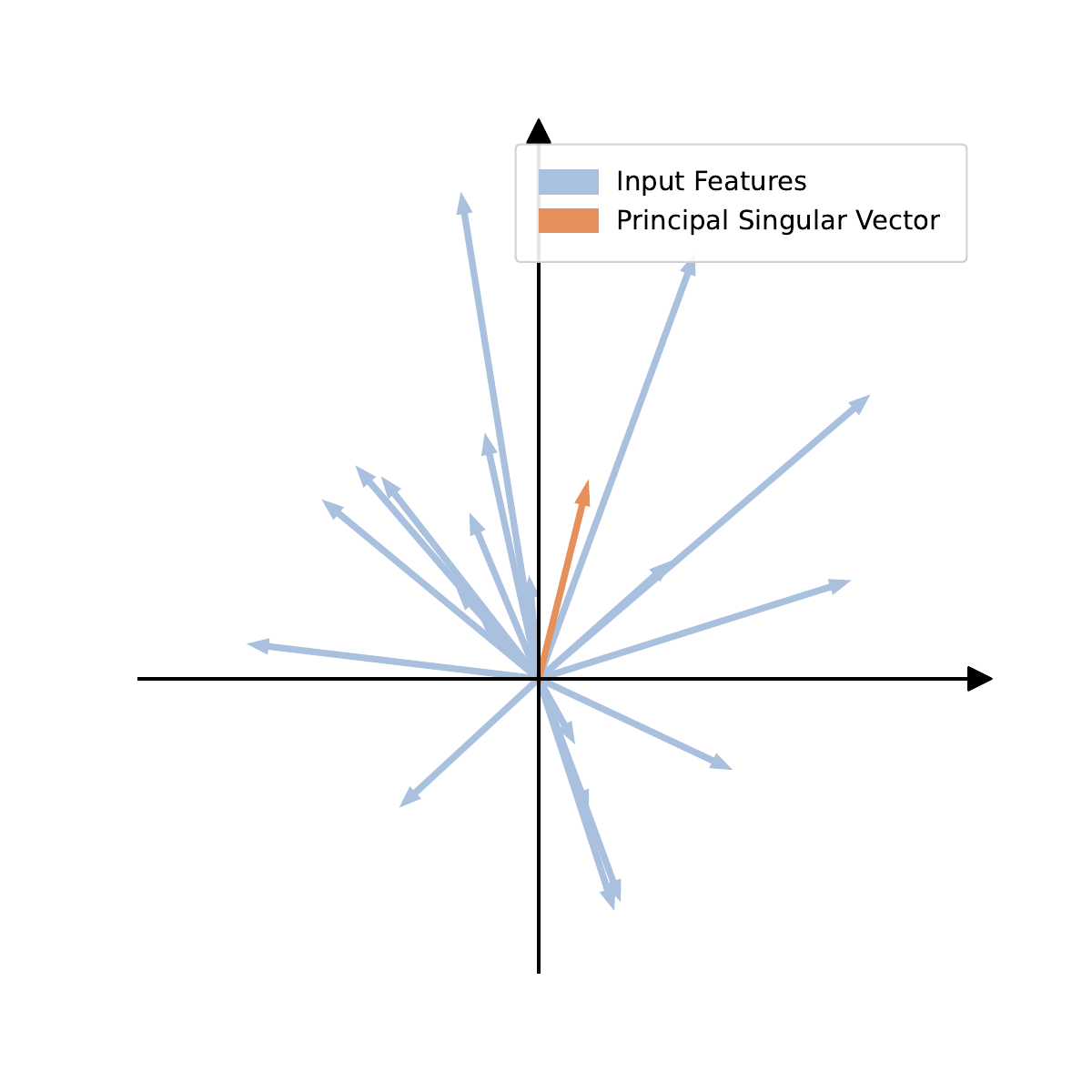}\label{fig:stable_sa}}
\subfigure[Exploding Loss]{\includegraphics[width=0.3\textwidth]{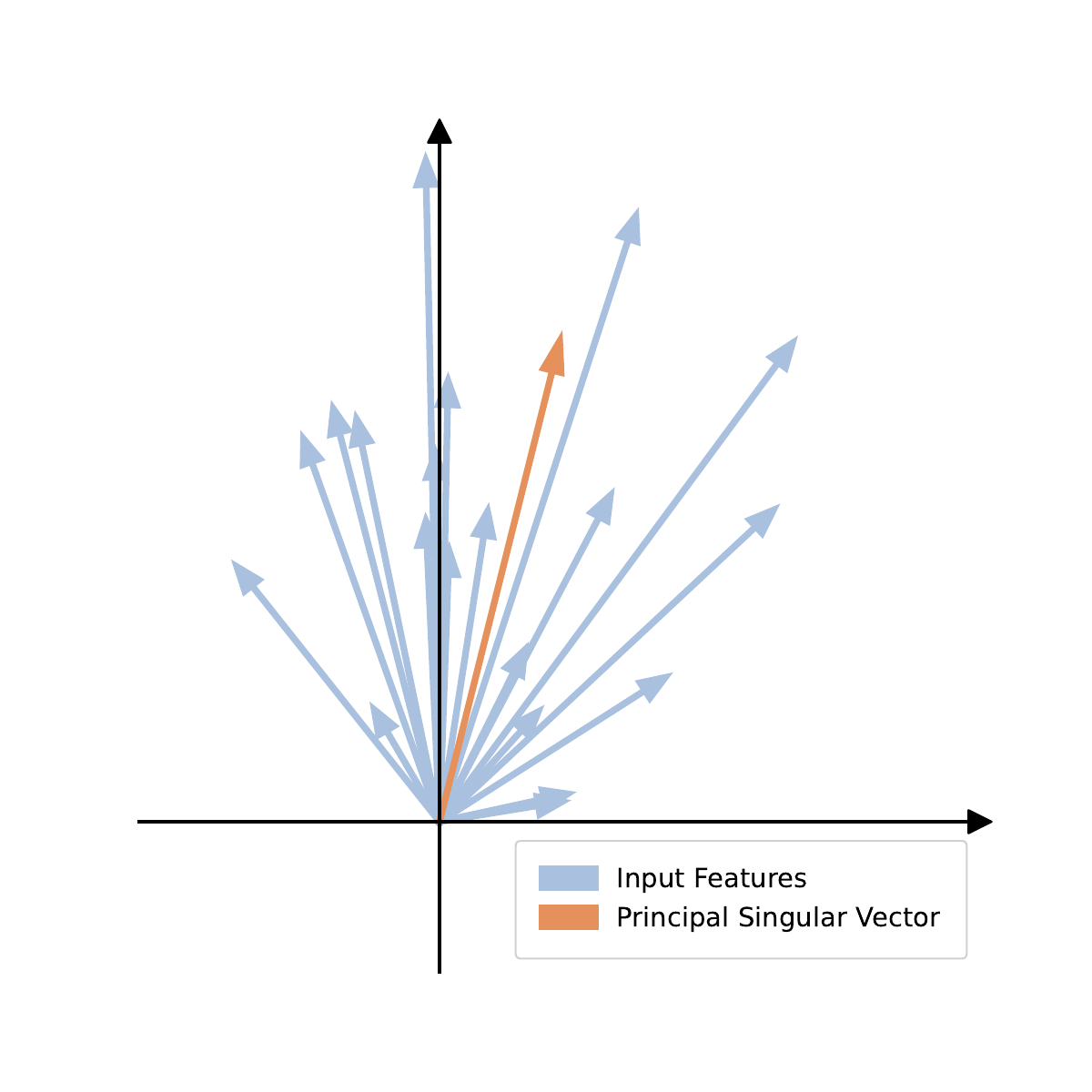}\label{fig:unstable_sa}}
\caption{Conceptual illustration of Spectral Alignment (SA). (a) SA measures the alignment between input features and the principal singular vector. (b) During stable training, inputs show diverse alignment directions. (c) Impending failure is signaled by a collapse in this diversity, with most inputs aligning in the same direction.}
\label{fig:conceptual_sa}
\end{figure*}

Therefore, we introduce the \emph{Spectral Alignment} to monitor the sign diversity. By calculating the spectral alignment of a batch of data, we obtain a distribution of alignment values that reflect how diversely the layer processes different inputs. Analyzing this distribution provides a clear picture of training stability: a balanced distribution around zero indicates a healthy layer, while a sudden collapse to one side signals impending training instability. Our contributions are as follows:

\begin{itemize}[leftmargin=*]
    \item We propose the spectral alignment and monitor its distribution as a proactive, mechanistically-grounded predictor for training instability.
    \item We provide the theoretical motivation for spectral alignment, connecting the loss of sign diversity in spectral alignment to the mechanisms of spectral norm increasing and training divergence.
    \item We demonstrate empirically on language models that spectral alignment detects nascent training failures significantly more clearly and earlier than conventional metrics, providing a crucial window for intervention.
\end{itemize}

\section{Related Work}
\subsection{Monitoring Training Stability in Deep Neural Networks}

The challenge of preventing training collapse has led to the development of various diagnostic metrics, 
most of which track scalar summaries of the training state. 

A primary category of indicators involves monitoring activations and norms,
as their explosive growth often precedes training failure. 
For instance, \citet{rybakov2024methods} observe that in a diverging Transformer, 
the norm of linear layer outputs can more than double relative to a stable run. 
Similarly, 
\citet{takase2025spike} find that catastrophic loss spikes are preceded by a sudden, 
uncontrolled growth in the global gradient norm.
Beyond monitoring norms, 
some studies also identify maximum activation values as predictive indicators. 
\citet{wortsman2023small} predict large model instability by monitoring 
the maximum value of attention logits, 
while \citet{fishman2024scaling} consider outlier activation spikes 
that appear late in training as a direct manifestation of training failure.

Other work has focused on more structural properties of the weight matrices, 
such as the stable rank (the ratio of the squared Frobenius norm to the squared spectral norm), 
which serves as a continuous proxy for matrix rank~\citep{Sanyal2020Stable}.
\citet{zhai2023sreparam} introduce attention entropy as a proxy for the stability of the attention mechanism, 
theoretically linking it to the spectral norms of the query and key projection matrices. 
Furthermore, \citet{nishida2024initialization} posit that loss spikes originate 
from non-uniformity in parameter updates, 
proposing the relative update ratio of parameters as a monitoring metric.

However, these scalar metrics are ambiguous and as early predictors of failure.
A rising value could signal healthy learning or impending divergence, 
often acting as a lagging indicator that confirms failure only after it becomes critical.
This underscores the need for more mechanistically grounded diagnostics.

\subsection{Spectral Properties for Model Training}

Training stability of deep networks is tightly connected to spectral properties of weight matrices and network Jacobians. For instance, \citet{pennington2017resurrecting} analyze dynamical isometry (keeping Jacobian singular values near~1) to explain and help avoid exponentially vanishing or exploding gradients. Several works use spectral constraints directly to improve stability. \citet{miyato2018spectral} introduce spectral normalization to constrain layer spectral norms and stabilize GAN training, a technique that also transfers to other architectures where weight conditioning matters. In recurrent models, orthogonal/unitary parametrizations (unitary RNNs) are proposed to preserve signal norm over many time steps and mitigate exploding / vanishing gradients~\citep{arjovsky2016unitary}. For modern deep Transformers, methods that change initialization / normalization to bound updates and smooth gradient norms have been proposed to enable very deep architectures (e.g., DeepNet and BranchNorm)\citep{wang2022deepnet,liu2023branchnorm}.

Spectral analysis has also been used to explain representation learning and feature emergence. \citet{imani2021representation} empirically show that learned representations' top singular vectors align with targets (representation alignment), and that alignment tends to increase with depth—connecting spectral structure of features to downstream transfer performance. More recently, \citet{yang2022tensor,yang2023spectral} formalize a spectral condition for feature learning, proving that correct spectral scaling of weight matrices (and their updates) is necessary for nontrivial feature learning at large widths and motivating parametrizations such as the Maximal Update Parametrization ($\mu$P).

While prior work uses spectral properties to enforce stability or explain learning, their potential as early-warning indicators for training instability remains underexplored. This work proposes a novel spectral indicator to fill this gap, enabling early detection and intervention.

\section{Spectral Alignment and Theoretical Justification}

This section introduces Spectral Alignment and its theoretical foundation as an early-warning predictor for loss explosion.
We ground our analysis in a general linear layer defined by a pre-activation $\mathbf{f}^{(l)} = \mathbf{h}^{(l-1)} \mathbf{W}_l \in \mathbb{R}^{1 \times n_l}$ and an element-wise activation function $\mathbf{h}^{(l)} = \phi(\mathbf{f}^{(l)})$. 
In \cref{sec:sai_definition}, we motivate and define the Spectral Alignment. Subsequently, in \cref{sec:theoretical_analysis}, we provide a rigorous theoretical justification for its effectiveness. In \cref{ssec:use_sai}, we discuss practical considerations for monitoring real-world training scenarios with Spectral Alignment.

\subsection{Definition of Spectral Alignment}\label{sec:sai_definition}

The foundation of SA rests on the spectral dynamics of feature learning. During training, the input representation to a layer, $\mathbf{h}^{(l-1)}$, tends to align with the principal left singular vector, $\mathbf{u}_1$, of its corresponding weight matrix, $\mathbf{W}_l$~\citep{imani2021representation,yang2023spectral} (see \cref{fig:sa_concept}). This alignment is the mechanism by which the layer's most significant learned linear transformation is applied to the input data.

While some degree of alignment is essential for feature learning, its distributional character across a batch of inputs is critical for stability. A healthy, discriminative layer should respond differently to varied inputs. Its principal singular vector, representing its dominant feature, should therefore align positively with some inputs and negatively with others (see \cref{fig:stable_sa}). A collapse in this sign diversity—where nearly all inputs are pushed in the same direction (i.e., their inner products with $\mathbf{u}_1$ share the same sign)—indicates the onset of representational collapse~\cref{fig:unstable_sa}. This loss of diversity indicates that the layer is no longer processing features discriminatively but has entered a dangerous amplification loop, which, as we prove in \cref{sec:theoretical_analysis}, leads to unchecked growth in the spectral norm and to training collapse.

To quantify this alignment, 
we propose the Spectral Alignment (SA) to measure the alignment between an input and the primary direction of transformation for that layer:

\begin{definition}[Spectral Alignment]
The Spectral Alignment (SA) for the $i$-th input sample to layer $l$ is the cosine similarity between the input vector $\mathbf{h}^{(l-1, i)}$ and the principal left singular vector $\mathbf{u}_1(\mathbf{W}_l)$ of the layer's weight matrix:
\begin{equation}
    \text{SA}_l^{(i)} = \frac{\langle \mathbf{h}^{(l-1, i)}, \mathbf{u}_1(\mathbf{W}_l) \rangle}{\|\mathbf{h}^{(l-1, i)}\|_2 \cdot \|\mathbf{u}_1(\mathbf{W}_l)\|_2}.
\end{equation}
\end{definition}

The key to the early warning lies not in any single SA value, but in the \emph{distribution} of these values across a data batch. A healthy layer should exhibit a distribution of $\{\text{SA}_l^{(i)}\}_{i \in S}$ with significant mass on both positive and negative sides of zero, reflecting its discriminative function. A collapse of this distribution, where most values assume the same sign, is a strong indicator of impending loss explosion.

\subsection{Theoretical Justification for Detection Ability of Spectral Alignment}
\label{sec:theoretical_analysis}

To understand why Spectral Alignment 
serves as a reliable early-warning indicator for training instability, 
we need to examine its theoretical foundation. 
This analysis clarifies the causal chain between representational 
collapse and loss explosion, 
establishing SA as a theoretically grounded predictor.

For simplicity and theoretical tractability, our analysis focuses on a standard Multi-Layer Perceptron (MLP). This model, composed of linear layers and non-linear activations, captures the essential dynamics of feature learning and weight updates that are fundamental to more complex architectures, including the Transformer models used in our experiments. By analyzing the MLP, we can isolate the core mechanism of spectral alignment without the confounding variables of attention or advanced normalization schemes, providing a clear and foundational understanding of the connections between spectral alignment and unstable training. The model is defined as follows: for $l=1\cdots L$, we have $\mathbf{h}^{(l)} = \mathrm{ReLU}(\mathbf{f}^{(l)})$ and $\mathbf{f}^{(l)} = \mathbf{h}^{(l-1)} \mathbf{W}_l$.

To quantify the degree of spectral alignment, we perform an orthogonal decomposition of the principal left singular vector $\mathbf{u}_1$ of the weight matrix $\mathbf{W}_l$, with respect to the input vector $\mathbf{h}^{(l-1)}$:
\begin{align}\label{eq:orthogonal_decomposition}
    \mathbf{u}_1^\top &= \alpha \mathbf{h}^{(l-1)} + \boldsymbol{\epsilon}.
\end{align}
Here, $\alpha$ is a positive projection coefficient because we can set $-\mathbf{u}_1$ and $+\mathbf{u}_1$ as the top principal left singular vector, and $\boldsymbol{\epsilon}$ is the residual term orthogonal to $\mathbf{h}^{(l-1)}$. 
Based on this orthogonal decomposition, 
we can formally state when spectral alignment is in a pathological state with the following assumption.
\begin{assumption}[Pathological Spectral Alignment]\label{ass:alignment}
If spectral alignment of a neural network layer is in a pathological state, the following inequality holds in \eqref{eq:orthogonal_decomposition}:
\begin{equation}\label{eq:pathological_alignment}
    \|\boldsymbol{\epsilon}\| \ll \alpha \|\mathbf{h}^{(l-1)}\|.
\end{equation}
\end{assumption}

Note that
\eqref{eq:pathological_alignment} means that the norm of the orthogonal term is much smaller than the norm of the projection term, thus the directions of $\mathbf{u}_1$ and $\mathbf{h}^{(l-1)}$ are nearly aligned (see \cref{fig:stable_sa,fig:unstable_sa}).

To demonstrate that the pathological alignment in \cref{ass:alignment} leads to training failure, we leverage the established link between loss explosions and the growth of weight spectral norms~\citep{MIT_DL_Blog_2023,rybakov2024methods}. We will show that the condition in \cref{ass:alignment} is a direct cause of this spectral norm increase, thereby establishing it as a precursor to training failure.

The main challenge is to create a tractable mathematical model of the complex, nonlinear training dynamics. 
We employ matrix perturbation theory to derive the change 
in spectral norm during a single gradient descent step, 
and under the pathological spectral alignment condition, 
we provide an approximate expression for this change. 
By analyzing the training dynamics, 
we determine the sign of the key term in the expression, 
ultimately proving that
the spectral norm of the weight matrix exhibits inevitable positive growth.

\begin{theorem}[Spectral Norm Growth During Training]\label{thm:spectral_growth}
    Under \cref{ass:alignment}, 
    the change $\Delta\|\mathbf{W}_l\|_2$ of the spectral norm for weight $\mathbf{W}_l$ in a single gradient descent step can be approximated as:
    $$
    \Delta\|\mathbf{W}_l\|_2 \approx -\eta\rho\alpha
    \frac{\langle \mathbf{v}_1^\top, \mathbf{f}^{(l)} \rangle}{\|\mathbf{f}^{(l)}\|_2}
    \|\mathbf{h}^{(l-1)}\|_2^2 (\mathbf{p}-\mathbf{t})(\mathbf{h}^{(L)})^\top
    $$
    where $\eta, \rho$ are positive scalars.
    During training,
    the spectral norm change $\Delta\|\mathbf{W}_l\|_2$ is positive, 
    indicating that the spectral norm of $\mathbf{W}_l$ will exhibit \textbf{positive growth}.
\end{theorem}

\cref{thm:spectral_growth} shows that pathological spectral alignment 
directly causes spectral norm growth. 
This growth then triggers the explosion of activations and gradients, 
which violates the principle of maintaining stable activations in deep learning,
ultimately leading to training failure. 
\begin{corollary}\label{cor:growth_amplification}
The spectral norm growth established in \cref{thm:spectral_growth} can trigger a growth in activations and gradients.
\end{corollary}

This result provides the theoretical justification for using spectral 
alignment as an early-warning predictor for loss explosions, 
positioning it as a more fundamental indicator than conventional metrics.

\subsection{Implementation of Spectral Alignment to Monitor Training}\label{ssec:use_sai}
The practical application of Spectral Alignment as a monitoring tool is straightforward. At regular intervals during training, for a selected layer, we sample a batch of input activations and the corresponding weight matrix. We then compute the principal left singular vector of the weight matrix. While a full Singular Value Decomposition (SVD) is computationally expensive ($O(\mathrm{min}(n_{l-1}, n_l) n_{l-1} n_l)$), this vector can be efficiently approximated using a few steps of the Power Iteration method ($O(n_{l-1}n_l)$ per iteration), adding minimal overhead. For each input sample in the batch, we calculate its Spectral Alignment, yielding a distribution of values for the current training step.

The key is to analyze the shape of this distribution. A healthy layer, processing features discriminatively, will produce a distribution centered around zero with significant sign diversity. An impending failure is signaled by a collapse in this diversity, where the distribution shifts and concentrates on either the positive or negative side. This qualitative change serves as a clear and early warning, allowing for timely intervention. For an even more lightweight approach, the necessary tensors can be periodically saved for offline computation, eliminating any burden on the training loop. This makes spectral alignment a highly practical and scalable tool for monitoring model training.

\begin{figure*}[t]
    \centering
    \subfigure[Loss explosion caused by FA]{\includegraphics[width=0.4\textwidth]{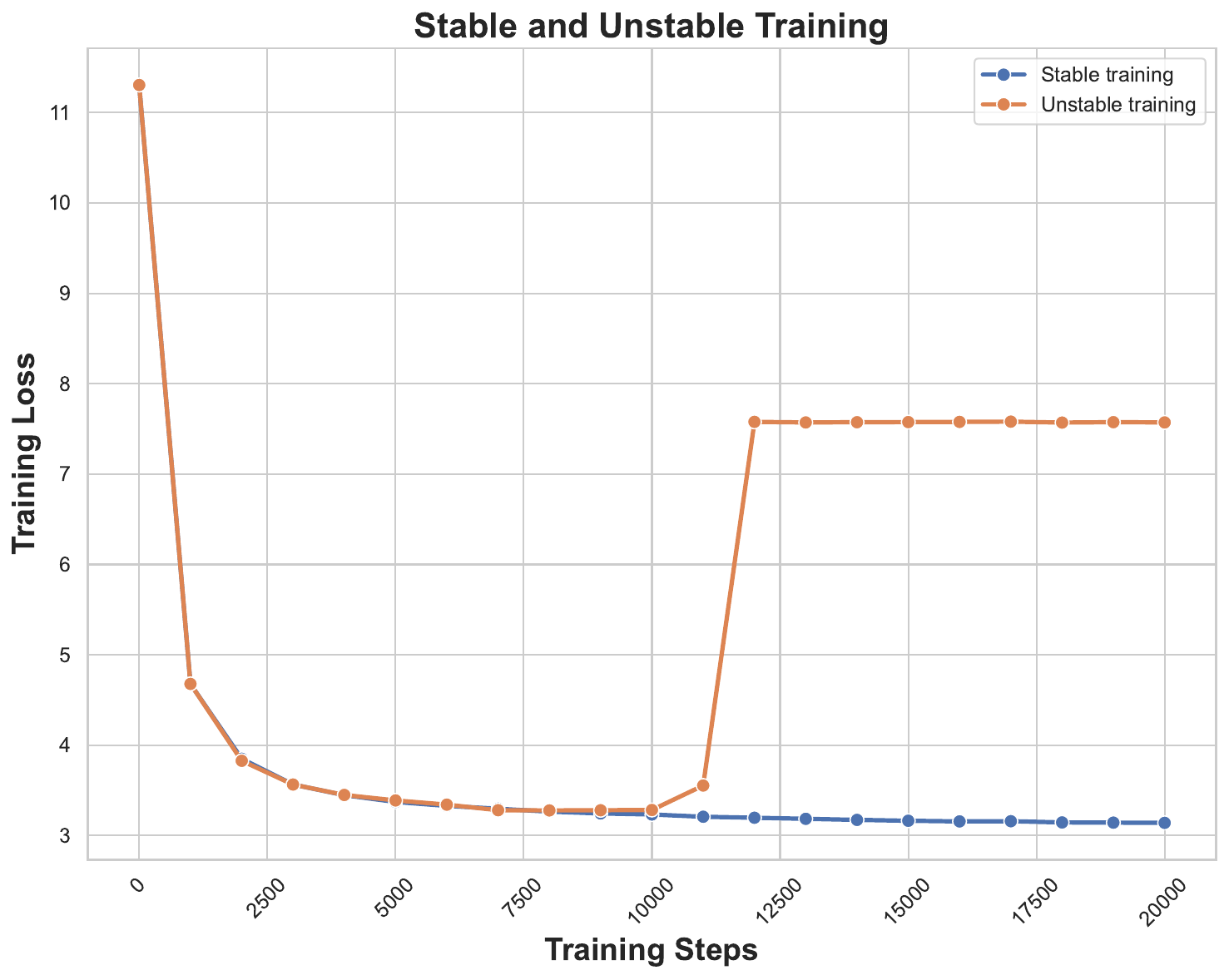}}\label{fig:fa_loss}
    \subfigure[Loss explosion caused by FFN]{\includegraphics[width=0.4\textwidth]{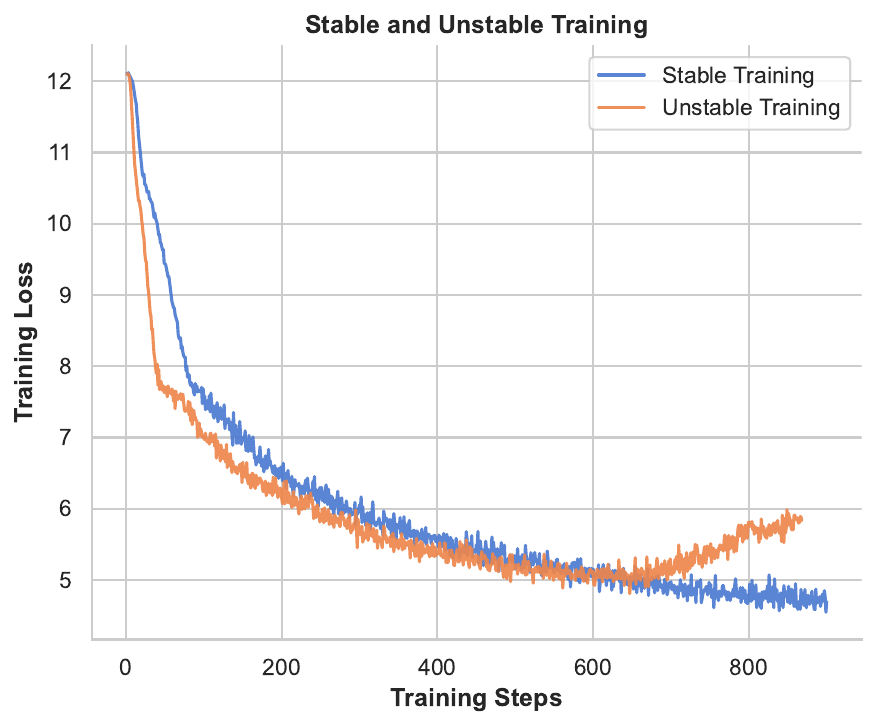}}\label{fig:ffn_loss}
    \caption{Comparison of stable and exploding training loss caused by FA and FFN.}\label{fig:fa_ffn_loss}
\end{figure*}

\paragraph{Comparison with Other Indicators}
Spectral Alignment serves as a more fundamental and direct indicator of instability. As our theoretical analysis in \cref{sec:theoretical_analysis} demonstrates, the growth in spectral norm is a direct \emph{consequence} of the collapse in alignment diversity. This norm growth then triggers a cascade of amplification in activations and gradients, ultimately causing loss explosions and training divergence~\citep{rybakov2024methods,fishman2024scaling,lee2024fp8}. While conventional scalar metrics provide ambiguous quantitative signals, the distribution of Spectral Alignment offers a clear, qualitative warning: a collapse from a diverse, balanced distribution to a one-sided one. This makes it a more reliable and consistent early-warning predictor.

\section{Experiments}

\subsection{Setup for Reproducing Loss Explosion}

\begin{figure*}[t]
    \centering
    \subfigure[SA distribution for stable training (FA)]{
        \includegraphics[width=0.4\textwidth]{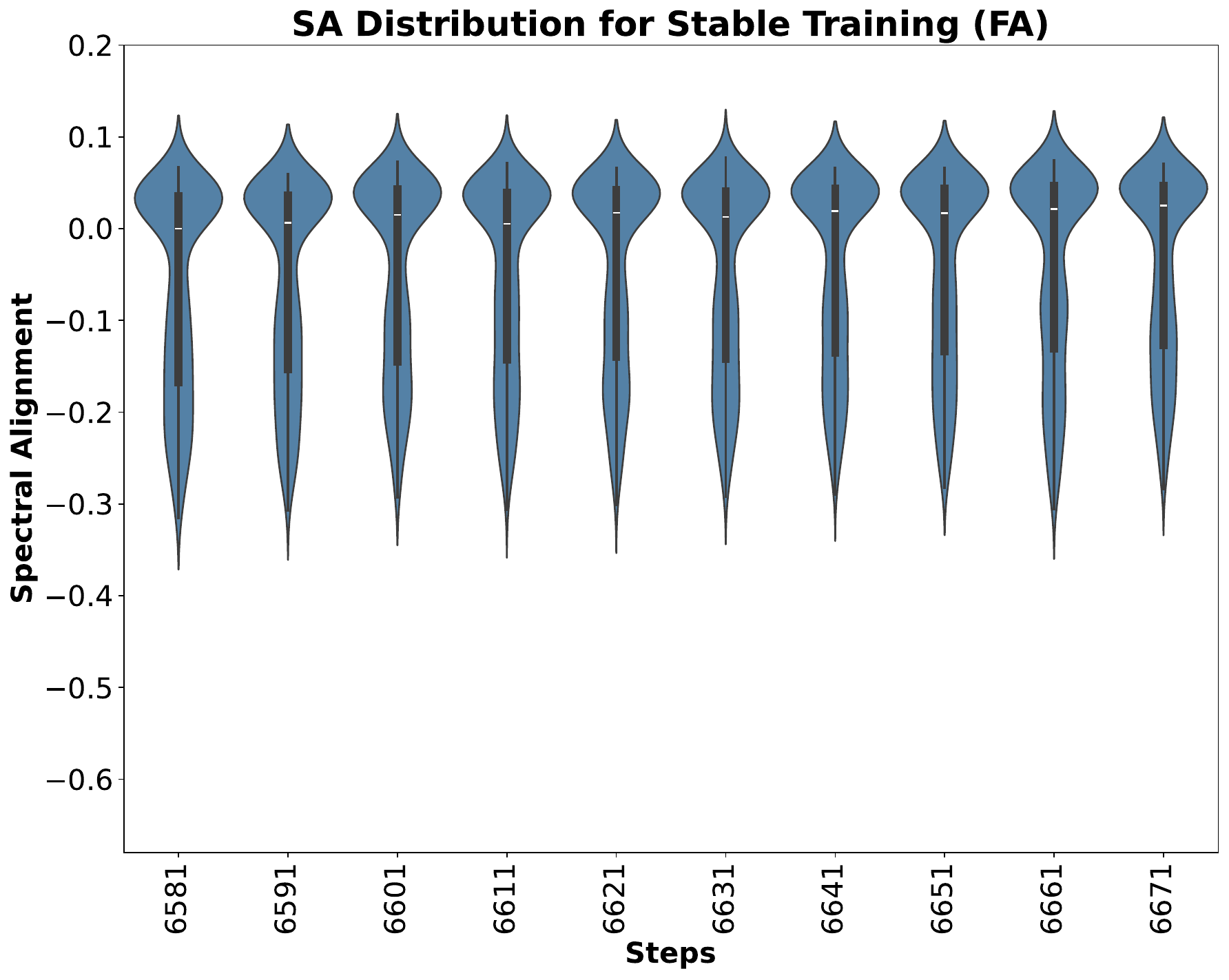}\label{fig:fa_sa_stable}
    }
    \subfigure[SA distribution for unstable training (FA)]{
        \includegraphics[width=0.4\textwidth]{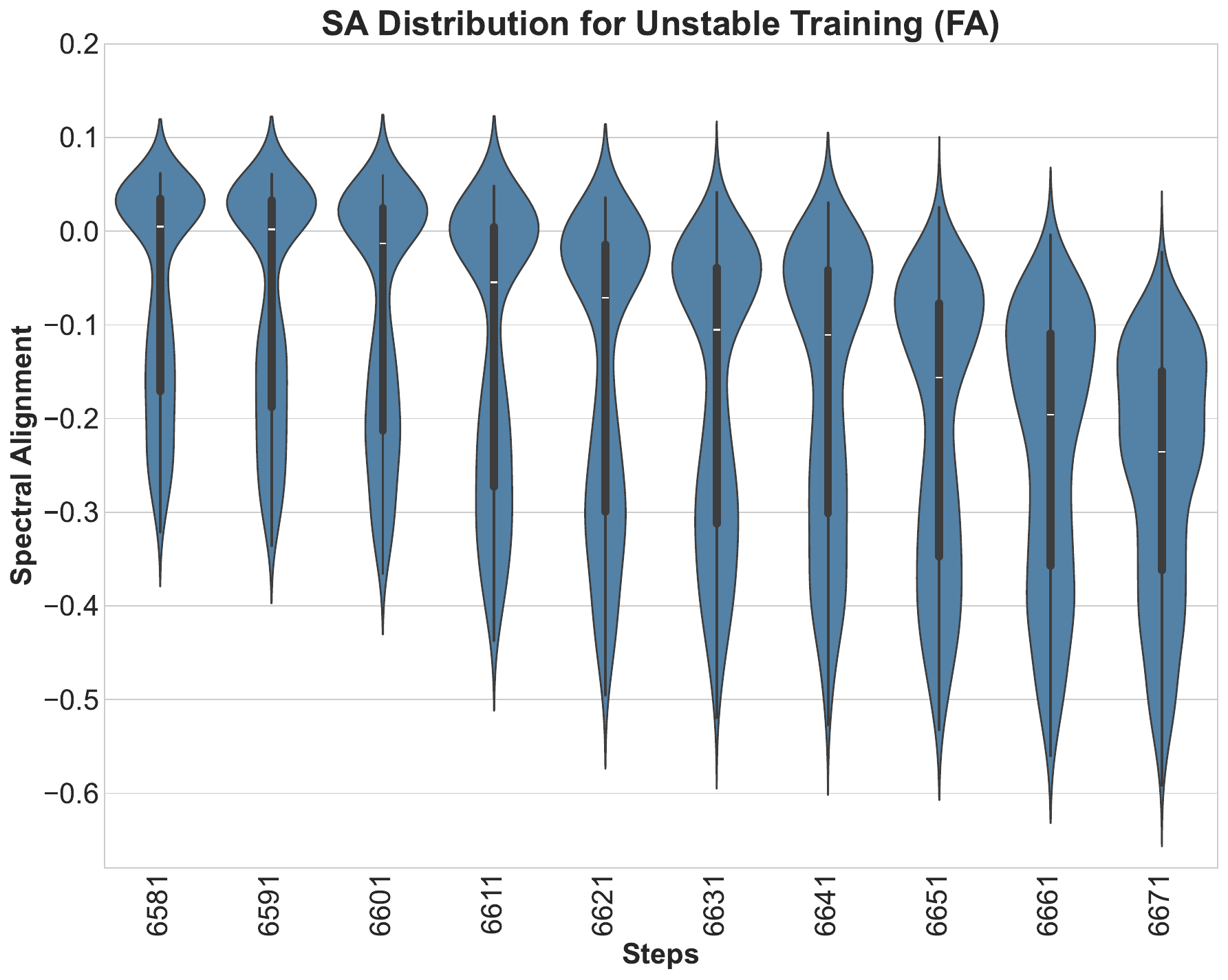}\label{fig:fa_sa_unstable}
    }
    \subfigure[SA distribution for stable training (FFN)]{
        \includegraphics[width=0.4\textwidth]{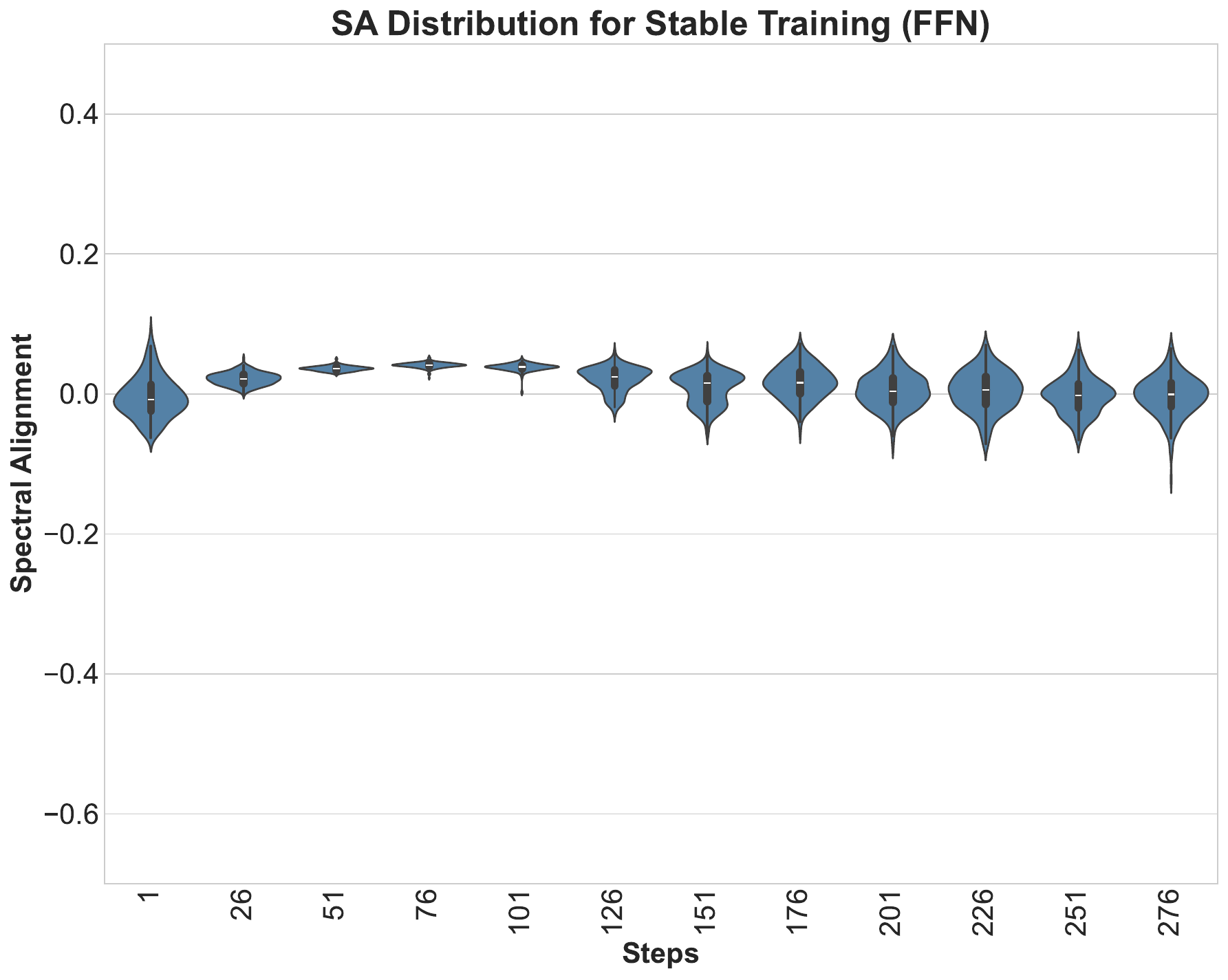}\label{fig:ffn_sa_stable}
    }
    \subfigure[SA distribution for unstable training (FFN)]{
        \includegraphics[width=0.4\textwidth]{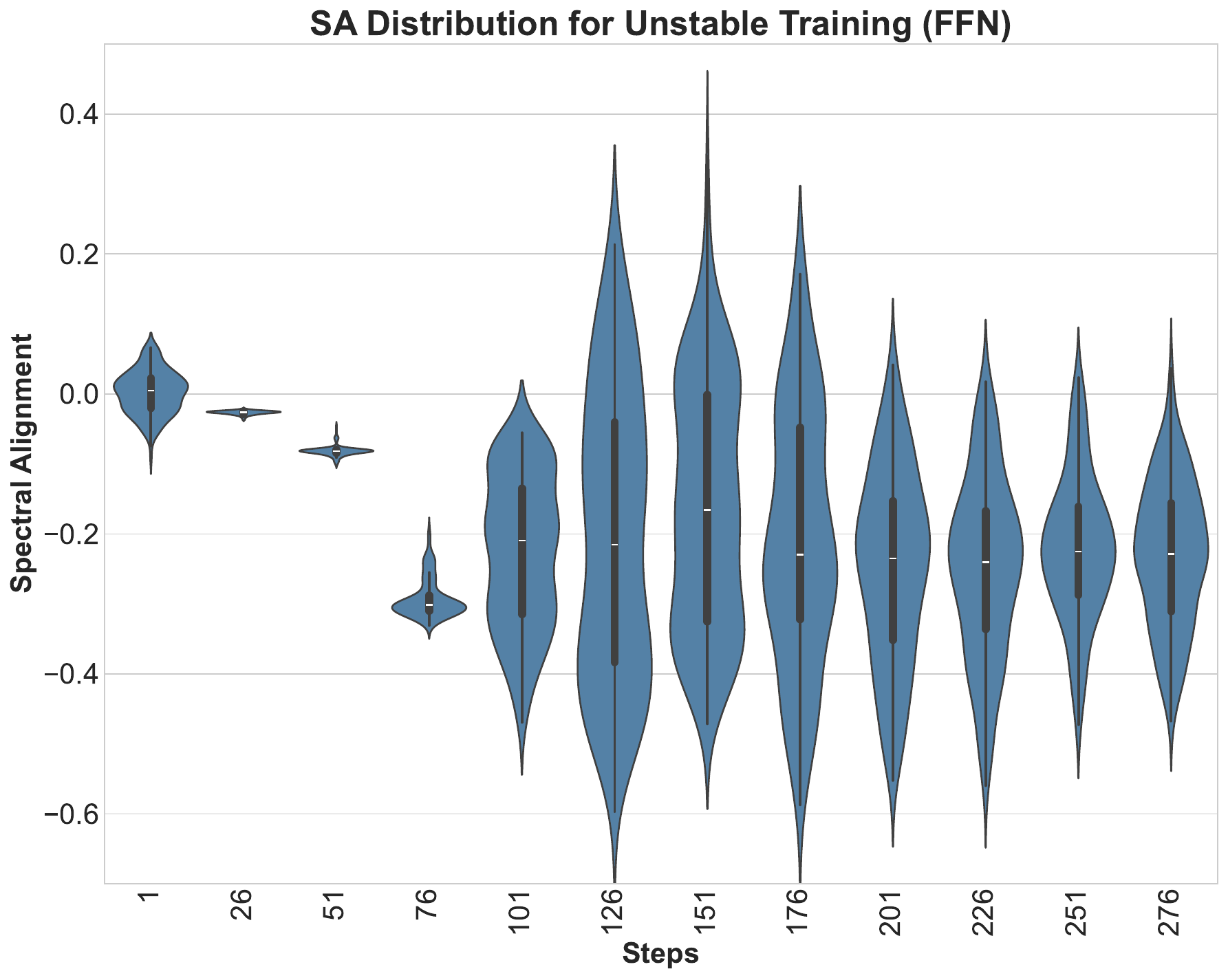}\label{fig:ffn_sa_unstable}
    }
    \caption{Spectral Alignment distributions for stable and unstable training.}\label{fig:sa_results}
\end{figure*}

\paragraph{Loss Explosion Caused by Flash Attention}
To investigate loss explosion, we reproduce a known instability issue that occurs when pre-training a GPT-2 model with Flash Attention in low precision~\citep{nanogpt_issue303,nanogpt_issue524}. Our experiments are implemented in PyTorch, utilizing the Distributed Data Parallel (DDP) framework for multi-GPU training.

We employ a 12-layer, 12-head decoder-only Transformer with a GPT-2 architecture and an embedding dimension of 768. The context length is set to 1024 tokens. To isolate the instability, dropout is disabled, and bias terms are removed from all LayerNorm and linear layers. Training is conducted using bfloat16 mixed precision.

The model is pre-trained on the OpenWebText dataset. We use the AdamW optimizer with a maximum learning rate of $1 \times 10^{-3}$, $\beta_1=0.9$, $\beta_2=0.95$, and no weight decay. The learning rate follows a cosine schedule with a 2000-step linear warmup, decaying to a minimum of $1 \times 10^{-5}$. Gradients are clipped at a maximum norm of 1.0.

The training is performed on 4 NVIDIA A100 GPUs (80GB). We use a per-GPU micro-batch size of 32 and 4 gradient accumulation steps, resulting in an effective batch size of 524,288 tokens. As shown in Figure \ref{fig:fa_loss}, this setup successfully reproduces the loss explosion, which occurs around step 11,000.

\paragraph{Loss Explosion Caused by FFN}
Following \citet{wortsman2023small}, who demonstrated that instabilities in large models can be replicated in smaller models by increasing the learning rate, we induce a loss explosion in a Qwen2.5-0.5B model. This allows us to study the failure dynamics in a controlled setting.

Our setup uses the Qwen2.5-0.5B model with a context length of 128 and BF16 mixed precision. The model is trained on the English portion of the Colossal Clean Crawled Corpus (C4), which is tokenized and divided into fixed-length chunks of 128 tokens. We use the AdamW optimizer with a weight decay of 0.1. To create stable and unstable runs, we set the learning rate to $1 \times 10^{-3}$ and $5 \times 10^{-3}$, respectively. The effective batch size is 128, achieved with a per-device batch size of 16 and 8 gradient accumulation steps. Training is managed using the Huggingface Accelerator framework.

As shown in \cref{fig:ffn_loss}, this configuration successfully reproduces the instability, with the high-learning-rate run experiencing a loss explosion around step 647. To analyze the failure, we use forward hooks to record the input activations and weights of the 10th layer's `gate\_proj' matrix within the SwiGLU FFN, which was identified as a source of instability via a preliminary spectral norm analysis.

\subsection{SA is an Unambiguous Predictor of Loss Explosion}

\paragraph{Compare SA in FA}
We first localize the source of instability by observing an abnormal increase in the spectral norm of the second-layer Flash Attention weights, which becomes apparent around step 8000. To test Spectral Alignment (SA) as an earlier predictor, we analyze its distribution for the key projection matrix in the same layer. \cref{fig:fa_sa_stable} and \cref{fig:fa_sa_unstable} visualize the SA distributions with violin plot from training steps 6581 to 6671, which is significantly earlier than the loss explosion around step 11,000. For the stable run (\cref{fig:fa_sa_stable}), the SA distribution remains centered, with a mean value consistently close to 0, indicating healthy sign diversity. In stark contrast, the unstable run (\cref{fig:fa_sa_unstable}) shows a dramatic collapse in sign diversity; the mean SA value rapidly becomes negative, dropping below -0.2 by step 6671. This provides a clear warning of impending failure thousands of steps before other indicators like spectral norm show significant anomalies.

\paragraph{Compare SA in FFN}
In the FFN instability experiment, we analyze the SA distribution for the 10th layer's gate\_proj matrix. \cref{fig:ffn_sa_stable,fig:ffn_sa_unstable} visualize the SA distributions with violin plot from training steps 1 to 276. For the stable run (\cref{fig:ffn_sa_stable}), the SA distribution remains centered, with a mean value consistently around 0, indicating healthy sign diversity. In stark contrast, the unstable run (\cref{fig:ffn_sa_unstable}) shows a dramatic collapse in sign diversity; the mean SA value rapidly becomes negative, dropping to around -0.2 by step 76, which is much earlier than the loss explosion. This provides a clear, early warning of impending failure.

\paragraph{Summary for Experiments of SA}
The key advantage of Spectral Alignment (SA) is its consistent and unambiguous signal across different failure setups, a stark contrast to the metrics discussed in \cref{ssec:other_indicators}. In both the Flash Attention and FFN instability experiments, SA provides an identical qualitative warning. Stable training runs consistently maintained an SA distribution centered around zero, reflecting healthy sign diversity. Conversely, unstable runs are preceded by a dramatic collapse of this distribution to one side, with its mean value rapidly shifting away from zero. This collapse serves as a powerful early warning, occurring thousands of steps before the loss explosion in the FA case and hundreds of steps earlier in the FFN case. This consistent, qualitative shift establishes SA as a reliable and superior monitoring tool.

\subsection{Other Indicators are Ambiguous for Detecting Loss Explosion}\label{ssec:other_indicators}

\paragraph{Spectral Norm of Weights}
As shown in \cref{fig:sn_results}, the spectral norm of weights is an ambiguous indicator. While the curves for stable and unstable training clearly diverge in retrospect, it is difficult to detect an anomaly from the unstable curve alone. The patterns of growth and the absolute values differ significantly across experimental setups. For instance, a spectral norm of 7.5 is stable in the FA experiment but precedes failure in the FFN experiment with a much lower value around 2.5. This lack of a consistent, universal signal means an operator cannot reliably identify an impending failure by observing a single training run without a stable baseline for comparison, making spectral norm an unreliable early-warning predictor.

\begin{figure}[h!]
    \centering
    \subfigure[Spectral norm of key weight (FA)]{
        \includegraphics[width=0.22\textwidth]{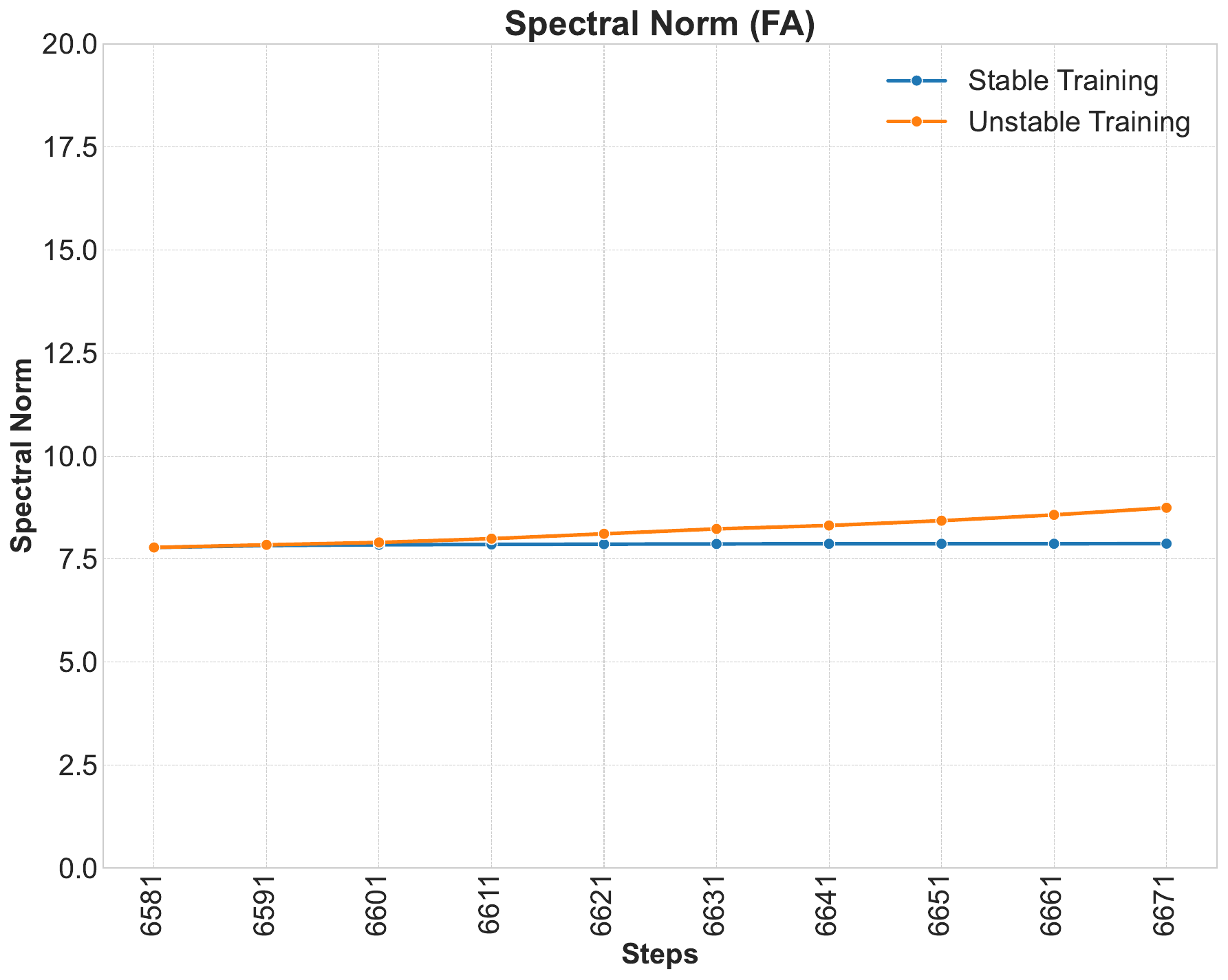}
    }
    \subfigure[Spectral norm of gate weight (FFN)]{
        \includegraphics[width=0.22\textwidth]{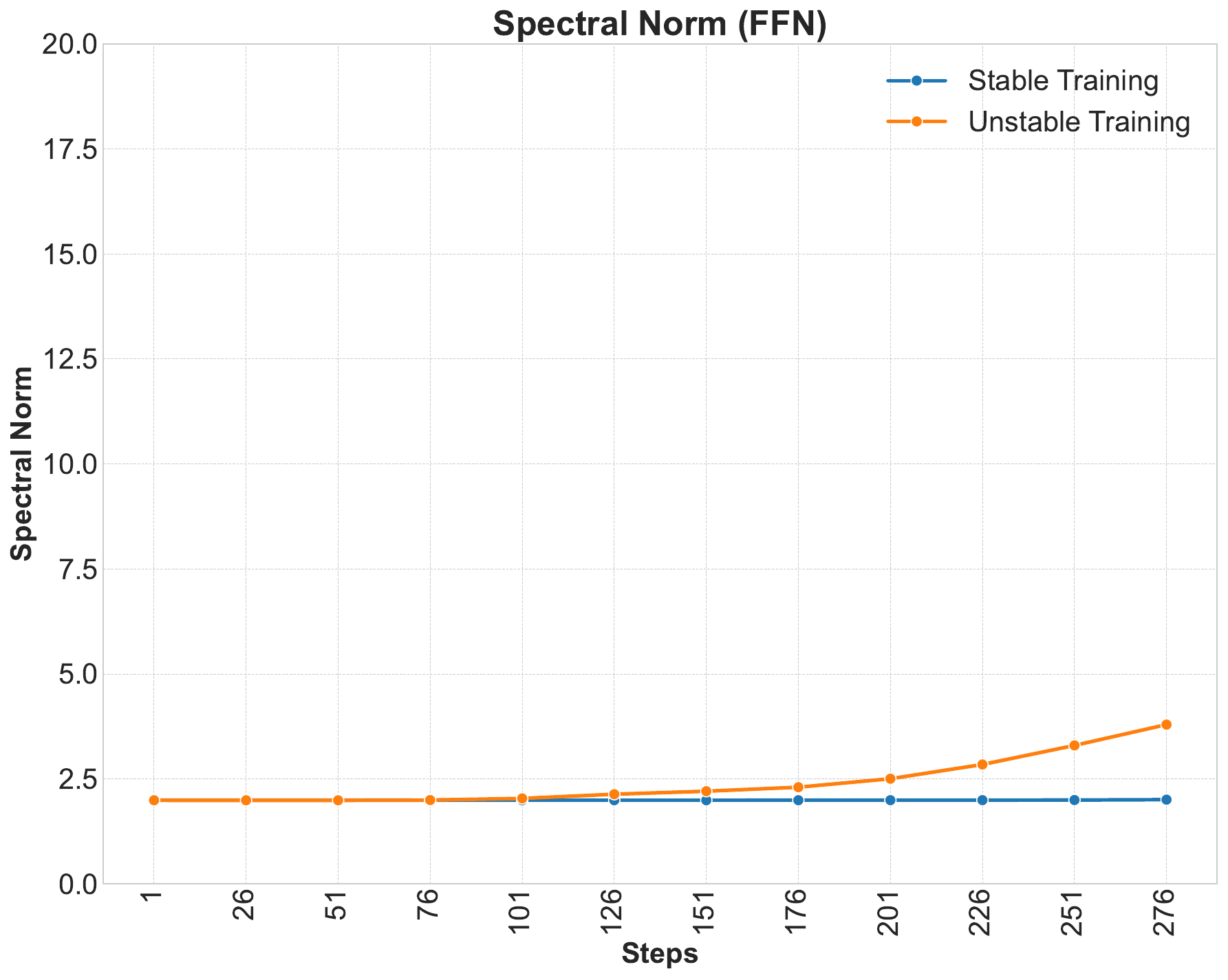}
    }
    \caption{Spectral norm of weights for stable and unstable training in two failed cases.}\label{fig:sn_results}
\end{figure}

\paragraph{Spectral Norm of Gradients}
The spectral norm of gradients is also an ambiguous and unreliable predictor, as shown in \cref{fig:sn_grad_results}. Its behavior is highly inconsistent across different failure modes. In the FA experiment, the gradient norm actually decreases from 0.08 to 0.01, masking the developing instability that SA detects. Conversely, in the FFN experiment, the gradient norm remains low (around 0.01) for both stable and unstable runs, offering no discernible signal of the impending failure. This lack of a consistent pattern in either absolute values or trends renders the gradient norm an unreliable tool for proactive monitoring.

\begin{figure}[h!]
    \centering
    \subfigure[Spectral norm of gradients for key weight (FA)]{
        \includegraphics[width=0.22\textwidth]{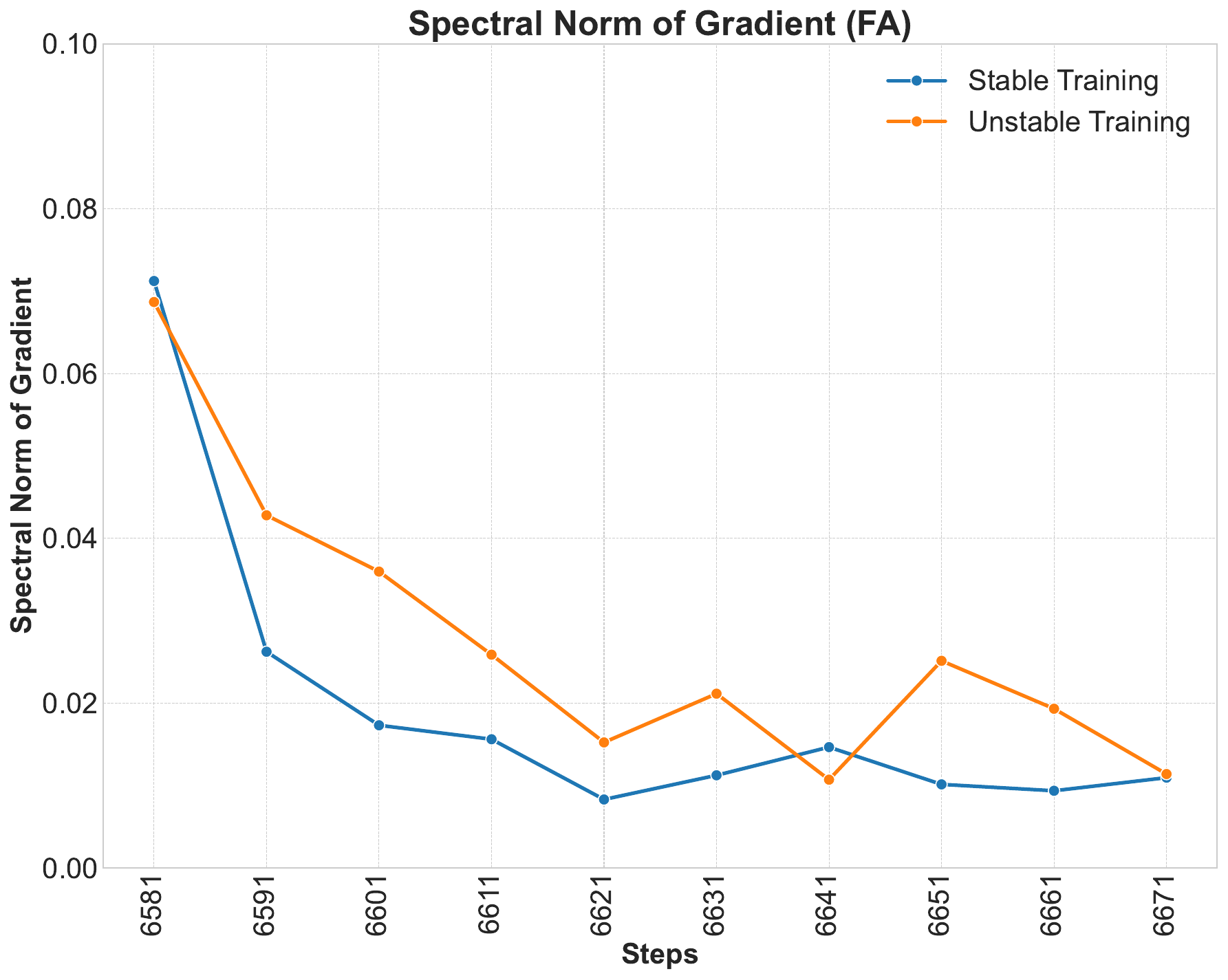}
    }
    \subfigure[Spectral norm of gradients for gate weight (FFN)]{
        \includegraphics[width=0.22\textwidth]{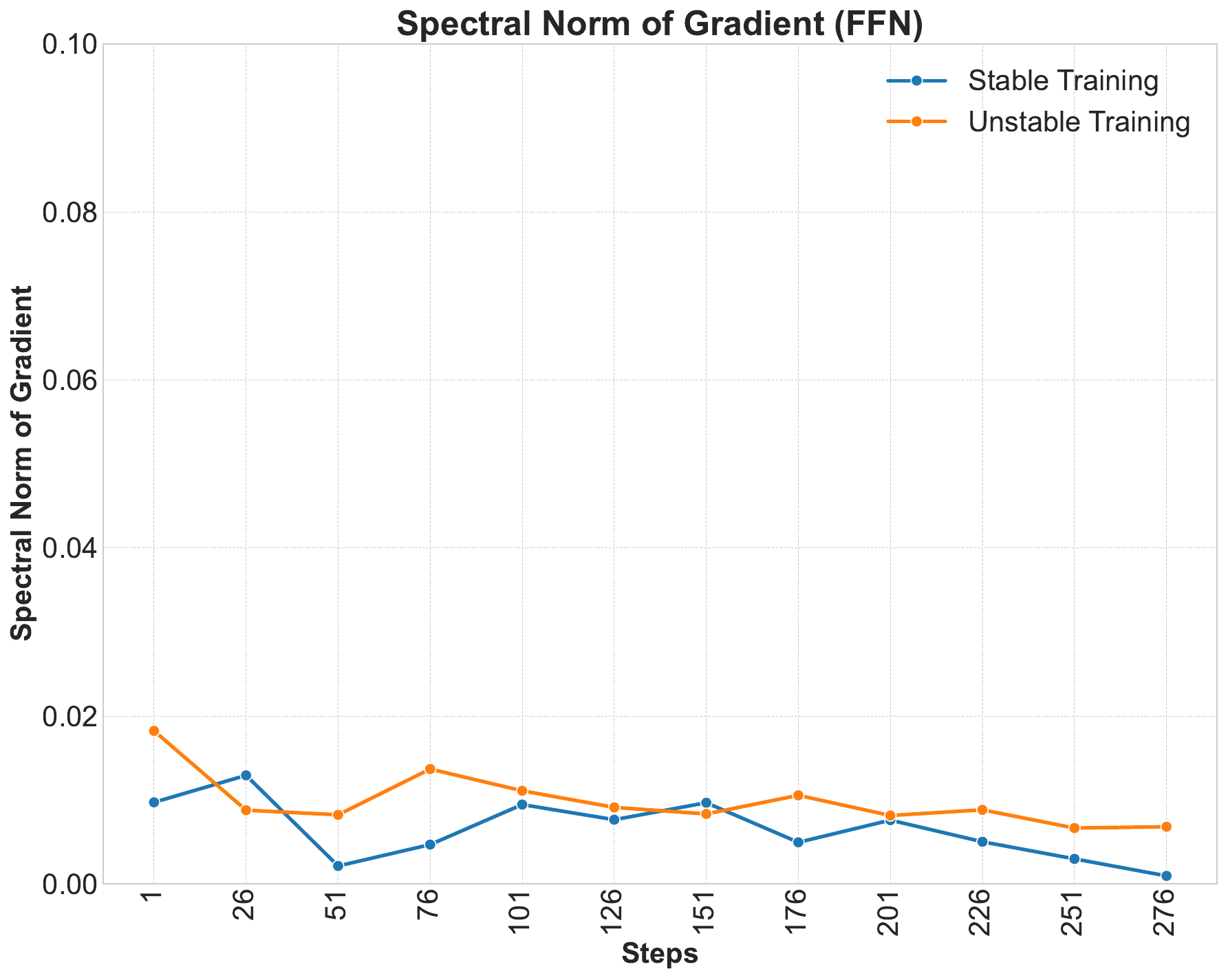}
    }
    \caption{Spectral norm of gradients for stable and unstable training in two failed cases.}\label{fig:sn_grad_results}
\end{figure}

\paragraph{Maximum Activation Value}
The maximum activation value is also an ambiguous and unreliable predictor, as shown in \cref{fig:max_activation_results}. Its behavior is highly inconsistent across different failure modes. In the FA experiment, the maximum activation for the unstable run diverges from the stable run early, but in the FFN experiment, the values for both runs remain low and close together until just before the loss explosion. This makes it impossible to set a universal threshold; an activation value of 20 is normal in the FA context but signals failure in the FFN case. Moreover, large activation values are not inherently problematic, as stable training runs can exhibit very high activations~\citep{sun2024massive}. This lack of a consistent scale or predictable pattern makes the maximum activation value a lagging and unreliable indicator of impending failure.

\begin{figure}[h!]
    \centering
    \subfigure[Maximum activation for key value (FA)]{
        \includegraphics[width=0.22\textwidth]{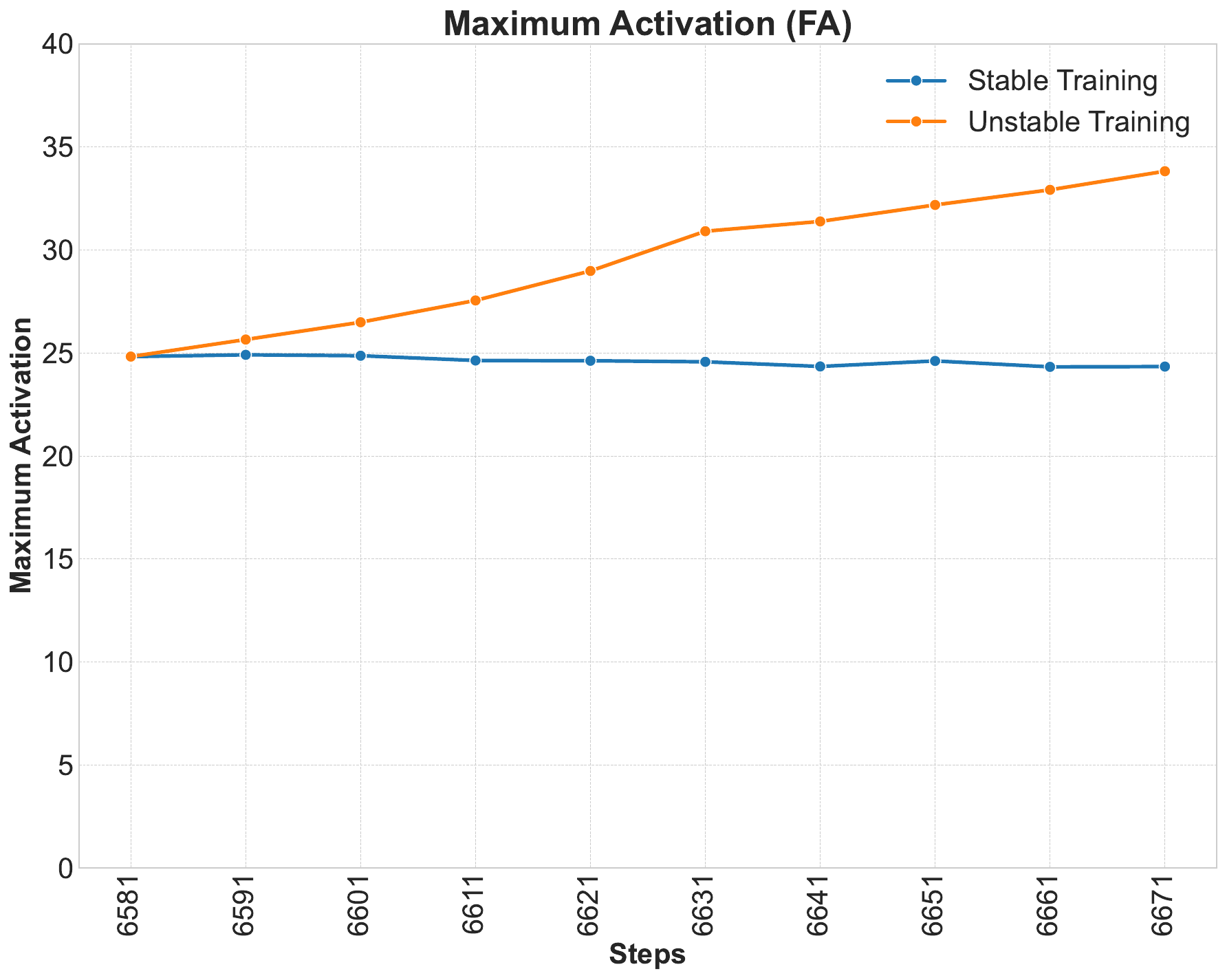}
    }
    \subfigure[Maximum activation for gate output (FFN)]{
        \includegraphics[width=0.22\textwidth]{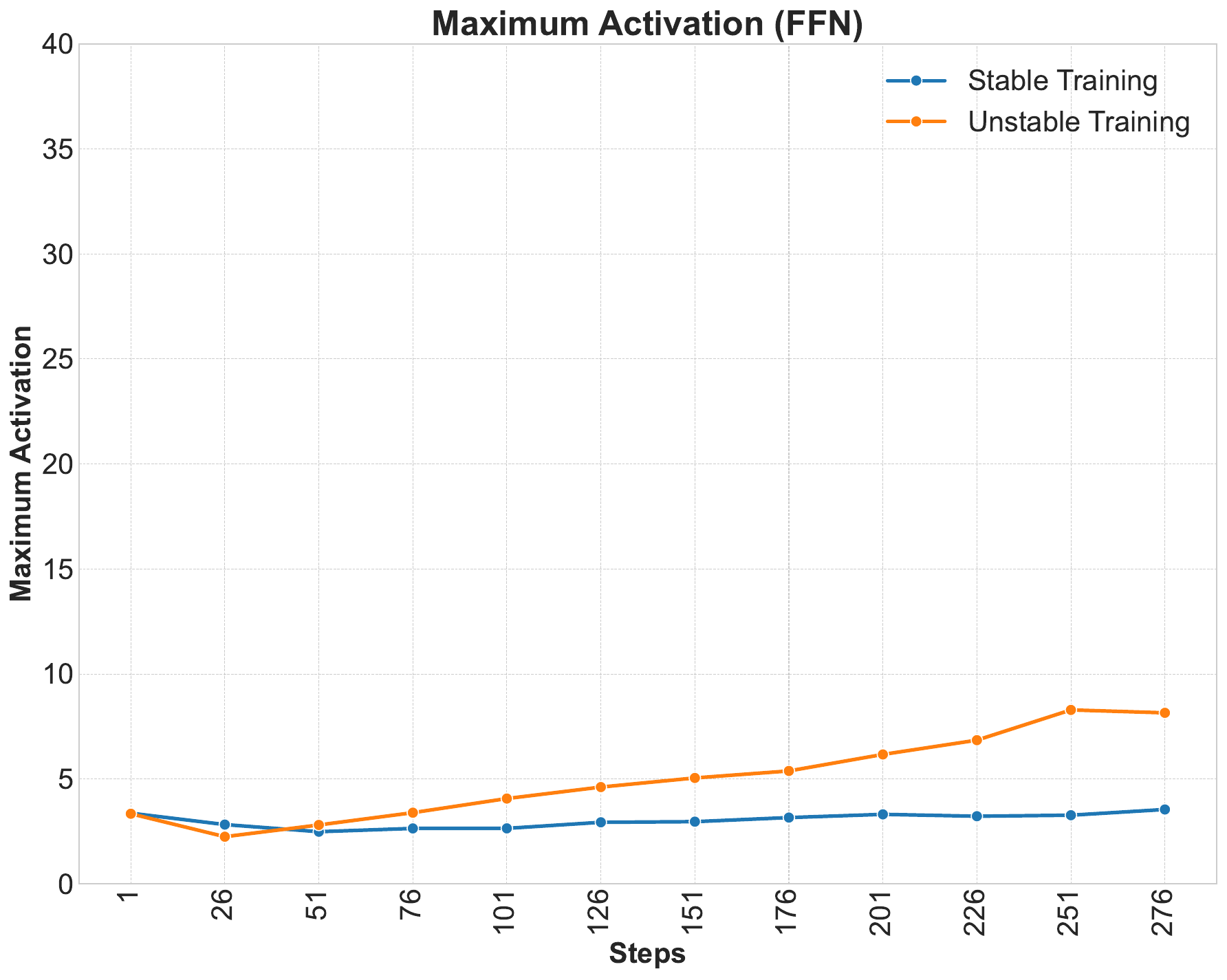}
    }
    \caption{Maximum activation values for stable and unstable training in two failed cases.}\label{fig:max_activation_results}
\end{figure}

\paragraph{Stable Rank of Weights}
The stable rank of weights is another unreliable predictor, as shown in \cref{fig:stable_rank_results}. Its behavior is highly inconsistent across different failure modes, making it impossible to establish a universal warning signal. In the FA experiment, the stable rank is low (around 9) for both stable and unstable runs, offering little distinction. In the FFN experiment, a notable contrast in stable rank behavior emerges between stable and unstable training. During unstable training, the stable rank maintains a high value of approximately 400 before dropping sharply just prior to model failure. Conversely, the stable FFN training session shows a gradual decrease in the stable rank, from around 437 to 433. This downward trend is similar to that observed in the unstable run of the FA experiment. This dramatic difference in both absolute scale and trend makes it impossible to define a consistent threshold or pattern for early detection, rendering stable rank an unreliable predictor for proactive monitoring.

\begin{figure}[h!]
    \centering
    \subfigure[Stable rank of key weight (FA)]{
        \includegraphics[width=0.22\textwidth]{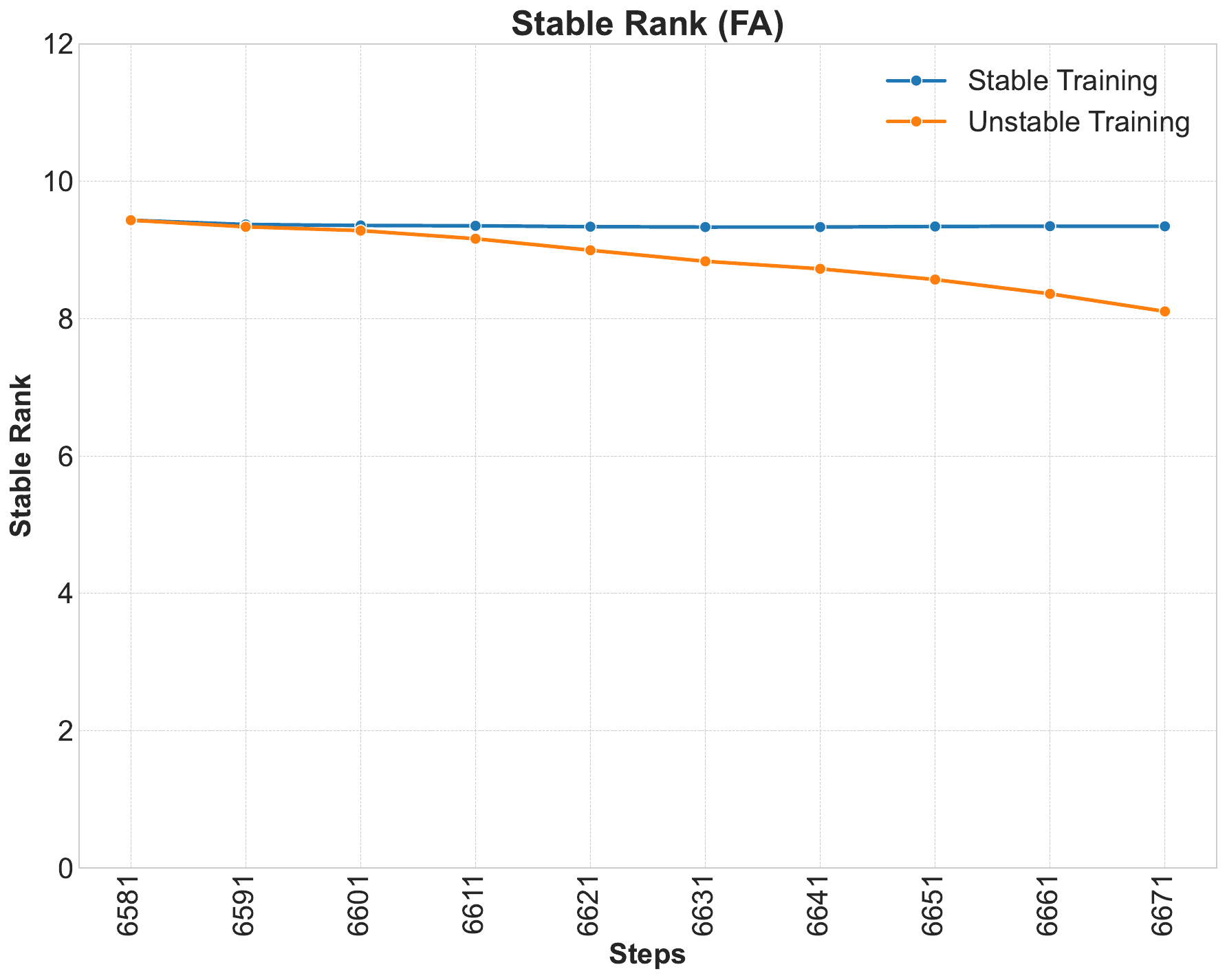}
    }
    \subfigure[Stable rank of gate weight (FFN)]{
        \includegraphics[width=0.22\textwidth]{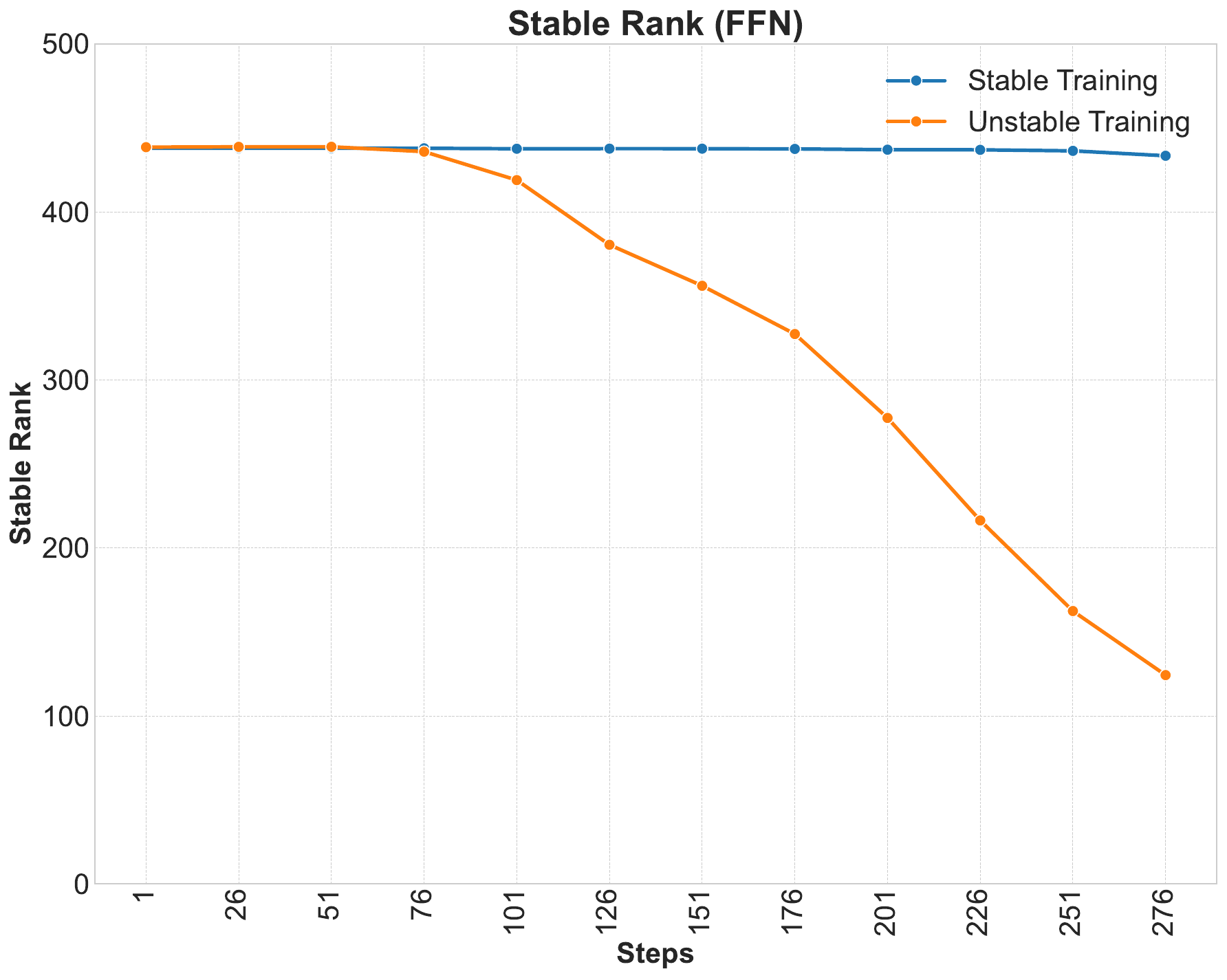}
    }
    \caption{Stable rank values for stable and unstable training in two failed cases.}\label{fig:stable_rank_results}
\end{figure}

\paragraph{Summary for Experiments of Other Indicators}
In contrast to the clear, qualitative signal from Spectral Alignment, conventional scalar metrics prove to be ambiguous and unreliable early-warning predictors. As shown in our experiments, metrics such as the spectral norm of weights (\cref{fig:sn_results}), spectral norm of gradients (\cref{fig:sn_grad_results}), maximum activation value (\cref{fig:max_activation_results}), and stable rank (\cref{fig:stable_rank_results}) all suffer from the same fundamental flaw: their absolute values and trends are highly dependent on the specific model architecture and failure mode. A value that signals impending failure in one context may be perfectly normal in another. This lack of a consistent, universal signal makes it practically impossible to set a reliable threshold for proactive monitoring without a stable baseline for comparison, often rendering them lagging indicators that only confirm a failure after it has become critical.

\section{Conclusion}
We introduced Spectral Alignment (SA), a novel, theoretically-grounded metric for the early detection of loss explosion. We demonstrate that a collapse in the sign diversity of the SA distribution is a fundamental precursor to training divergence, offering a significantly earlier and clearer warning than conventional metrics like weight or gradient norms. Its low computational overhead makes SA a practical tool for safeguarding resource-intensive model training.

\paragraph{Limitations}
Our work has several limitations. First, our theory is grounded in MLPs with ReLU activations, and while the empirical results on Transformers are strong, a formal theoretical extension to more complex architectures with mechanisms like attention and LayerNorm remains an area for future work. Second, while we demonstrated SA's effectiveness in two distinct and common failure scenarios, its effectiveness across all possible failure modes requires further investigation.


\clearpage

\nocite{*} 
\bibliography{ref}

\clearpage

\clearpage
\appendix
\thispagestyle{empty}

\onecolumn
\aistatstitle{Supplementary Materials}

\section{Proofs of Theoretical Results}

\begin{definition}
Consider an $L$-layer MLP with ReLU activation functions:
\begin{itemize}
    \item Input: $\mathbf{h}^{(0)} = \mathbf{x} \in \mathbb{R}^{1 \times n_0}$
    \item Hidden Layers: $\mathbf{h}^{(l)} = \mathsf{ReLU}(\mathbf{f}^{(l)})$, where $\mathbf{f}^{(l)} = \mathbf{h}^{(l-1)} \mathbf{W}_l$ for $l = 1, \dots, L$.
    \item $\mathbf{W}_l \in \mathbb{R}^{n_{l-1} \times n_l}$ is the weight matrix for layer $l$.
    \item $\mathbf{h}^{(l)} \in \mathbb{R}^{1 \times n_l}$ is the activation vector of layer $l$. (We use row vectors for activations).
    \item $\mathbf{f}^{(l)} \in \mathbb{R}^{1 \times n_l}$ is the pre-activation vector of layer $l$.
    \item Output: $\mathbf{h}^{(L)} \in \mathbb{R}^{1 \times n_L}$.
\end{itemize}

\end{definition}

\subsection{Proof with respect to parameters}

First we derive the gradient expression of the weight matrix $\mathbf{W}_l$ 
under the cross-entropy loss function, 
which is given by the following theorem:

\begin{theorem}[Gradient Expression for Linear Layers]\label{thm:gradient_expression}
Given an $L$-layer Multilayer Perceptron (MLP) with the ReLU activation function, i.e., $\mathbf{h}^{(l)} = \text{ReLU}(\mathbf{h}^{(l-1)} \mathbf{W}_l)$,
where $\mathbf{h}^{(0)} = \mathbf{x}, l=1 \dots L$. Let $\mathcal{L}$ be a cross-entropy loss function that depends on the final output $\mathbf{h}^{(L)}$.
Then, the gradient of the loss function $\mathcal{L}$ with respect to the weight matrix $\mathbf{W}_l$ can be approximated as:
$$
\frac{\partial \mathcal{L}}{\partial \mathbf{W}_l} \approx (\mathbf{h}^{(l-1)})^\top (\mathbf{p}-\mathbf{t})\rho \left( \prod_{k=L}^{l+1} \mathbf{W}_k^\top \right) \cdot \text{Diag}(\mathds{1}(\mathbf{f}^{(l)} > 0))
$$
where $\mathbf{p}-\mathbf{t}$ is a term related to the gradient of the loss function at the output layer,
and $\rho$ is a scalar approximation factor.
\end{theorem}

\begin{proof}
    By the chain rule of differentiation and matrix differentiation rules, the gradient of the weight matrix is given by:
    $$
    \frac{\partial \mathcal{L}}{\partial \mathbf{W}_l} = \left(\frac{\partial \mathbf{f}^{(l)}}{\partial \mathbf{W}_l}\right)^\top \frac{\partial \mathcal{L}}{\partial \mathbf{f}^{(l)}} = (\mathbf{h}^{(l-1)})^\top \frac{\partial \mathcal{L}}{\partial \mathbf{f}^{(l)}}.
    $$
    where the second equality is obtained by differentiating the linear layer's forward propagation expression $\mathbf{f}^{(l)} = \mathbf{h}^{(l-1)} \mathbf{W}_l$ with respect to $\mathbf{W}_l$.

    To find $\frac{\partial \mathcal{L}}{\partial \mathbf{f}^{(l)}}$,
    the gradient must be backpropagated from the network output $h^{(L)}$ to the $l$-th layer along the computation graph. This process also follows the chain rule:
    $$
    \frac{\partial \mathcal{L}}{\partial \mathbf{f}^{(l)}} = \frac{\partial \mathcal{L}}{\partial \mathbf{h}^{(L)}} \frac{\partial \mathbf{h}^{(L)}}{\partial \mathbf{f}^{(l)}} = (\mathbf{p} - \mathbf{t}) \frac{\partial \mathbf{h}^{(L)}}{\partial \mathbf{f}^{(l)}}.
    $$
    where the second equality holds because $\mathcal{L}$ is a cross-entropy loss function.
    Therefore, its gradient with respect to the final network output $\mathbf{h}^{(L)}$ (i.e., the logits) is $ \frac{\partial \mathcal{L}}{\partial \mathbf{h}^{(L)}} = \mathbf{p} - \mathbf{t} $, where $\mathbf{p}$ is the probability distribution row vector output by the network via the Softmax function, and $\mathbf{t}$ is the one-hot encoded row vector corresponding to the true label.

    The second term in the expression for $\frac{\partial \mathcal{L}}{\partial \mathbf{f}^{(l)}}$,
    $\frac{\partial h^{(L)}}{\partial f^{(l)}} \in \mathbb{R}^{n_L \times n_l}$, is a Jacobian matrix that describes the local linear behavior of the nonlinear mapping from the pre-activation of the $l$-th layer to the final network output. It can be found by the chain rule for composite functions, expressed as a product of the Jacobian matrices of each layer from layer $l+1$ to layer $L$:
    $$
    \frac{\partial \mathbf{h}^{(L)}}{\partial \mathbf{f}^{(l)}} = \left( \frac{\partial \mathbf{h}^{(L)}}{\partial \mathbf{h}^{(L-1)}} \right) \cdots \left( \frac{\partial \mathbf{h}^{(l+1)}}{\partial \mathbf{h}^{(l)}} \right) \left( \frac{\partial \mathbf{h}^{(l)}}{\partial \mathbf{f}^{(l)}} \right) = \left( \prod_{k=l+1}^{L} \frac{\partial \mathbf{h}^{(k)}}{\partial \mathbf{h}^{(k-1)}} \right) \frac{\partial \mathbf{h}^{(l)}}{\partial \mathbf{f}^{(l)}}
    $$

    Since the activation function is ReLU, $\mathbf{h}^{(l)} = \text{ReLU}(\mathbf{f}^{(l)})$, its corresponding Jacobian matrix is a diagonal matrix:
    $$
    \frac{\partial \mathbf{h}^{(l)}}{\partial \mathbf{f}^{(l)}} = \text{Diag}(\mathds{1}(\mathbf{f}^{(l)} > 0))
    $$
    Its diagonal elements are determined by the sign of the pre-activation values of that layer, being either 1 or 0. Similarly, by the chain rule,
    the intra-layer Jacobian of network layer $k$, $\frac{\partial \mathbf{h}^{(k)}}{\partial \mathbf{h}^{(k-1)}}$, is:
    $$
    \frac{\partial \mathbf{h}^{(k)}}{\partial \mathbf{h}^{(k-1)}} = \frac{\partial \mathbf{f}^{(k)}}{\partial \mathbf{h}^{(k-1)}} \frac{\partial \mathbf{h}^{(k)}}{\partial \mathbf{f}^{(k)}} = \mathbf{W}_k \text{Diag}(\mathds{1}(\mathbf{f}^{(k)} > 0))
    $$

    Thus, we have obtained the exact expression for the gradient:
    \begin{equation}\label{eq:gradient_exact}
    \frac{\partial \mathcal{L}}{\partial \mathbf{W}_l} = (\mathbf{h}^{(l-1)})^\top (\mathbf{p} - \mathbf{t}) \left( \prod_{k=L}^{l+1} \mathbf{W}_k^\top \text{Diag}(\mathds{1}(\mathbf{f}^{(k)} > 0)) \right) \text{Diag}(\mathds{1}(\mathbf{f}^{(l)} > 0))
    \end{equation}

    The above equation contains complex Diag terms that depend on the instantaneous activation state,
    making the Jacobian matrix a product of a series of random matrices, which is difficult to analyze directly.
    To obtain an expression that is more tractable for theoretical analysis, we adopt a mean-field approximation method, which is common in deep learning theory.

    Consider any scalar path from a neuron in the $l$-th layer to a neuron in the $L$-th layer in the computation graph. The validity of this path depends on whether all ReLU units along the way are activated.
    Assuming the network is sufficiently deep and the parameters are randomly initialized,
    it can be argued that this series of random, data-dependent multiplicative gating effects
    $\prod_{k=L}^{l+1} \text{Diag}(\mathds{1}(\mathbf{f}^{(k)} > 0))$
    can be replaced by its mean effect.

    We consider the dynamics of the network as follows: during training, for any such path, the probability of it being activated is the same, denoted by a scalar constant $\rho \in [0,1]$.
    This consideration reflects the mean-field approximation viewpoint, i.e., that the validity of the path for gradient signal propagation is uniform throughout the network.

    Based on this, we approximate the complex, instantaneous activation-dependent Diag matrices in the exact expression (\ref{eq:gradient_exact}),
    modeling their overall effect as a scalar factor $\rho$ times the product of the transposed weights:
    \begin{equation}\label{eq:gradient_approx}
    \left( \prod_{k=L}^{l+1} \mathbf{W}_k^\top \text{Diag}(\mathds{1}(\mathbf{f}^{(k)} > 0)) \right) \text{Diag}(\mathds{1}(\mathbf{f}^{(l)} > 0)) \approx \rho \left( \prod_{k=L}^{l+1} \mathbf{W}_k^\top \right)
    \end{equation}
    Here, $\rho$ is defined as the expected probability of any scalar path from layer $l$ to layer $L$ being activated, and the $\prod$ symbol denotes the continued product of matrices. The above approximation captures the average strength of gradient propagation while ignoring sample-specific activation fluctuations.

    Substituting the approximation (\ref{eq:gradient_approx}) into the expression (\ref{eq:gradient_exact}),
    and retaining the activation term of the current layer $\text{Diag}(\mathbf{1}(\mathbf{f}^{(l)} > 0))$, we obtain the final approximate gradient expression:
    $$
    \frac{\partial \mathcal{L}}{\partial \mathbf{W}_l} \approx (\mathbf{h}^{(l-1)})^\top (\mathbf{p}-\mathbf{t})\rho \left( \prod_{k=L}^{l+1} \mathbf{W}_k^\top \right) \cdot \text{Diag}(\mathds{1}(\mathbf{f}^{(l)} > 0))
    $$
\end{proof}

For a linear layer $\mathbf{f}^{(l)} = \mathbf{h}^{(l-1)} \mathbf{W}_l$
in a deep neural network, 
we now derive the condition under which the spectral norm $\|\mathbf{W}_l\|_2$ increases after a parameter update.

Under gradient descent, the update rule for the weight matrix $\mathbf{W}_l$ is:
$$
\mathbf{W}_{l, \text{new}} = \mathbf{W}_l - \eta \frac{\partial \mathcal{L}}{\partial \mathbf{W}_l}
$$
To analyze the change in the spectral norm of $\mathbf{W}_l$ after this update, 
we rely on matrix perturbation theory. 
The updated matrix can be seen as the original matrix $\mathbf{W}_l$ 
plus a small perturbation $-\eta \frac{\partial \mathcal{L}}{\partial \mathbf{W}_l}$. 
The following theorem describes the spectral norm increasing under such a perturbation.

\begin{theorem}[Change in Spectral Norm under Perturbation]\label{thm:spectral_norm_change}
    Let $A \in \mathbb{R}^{m \times n}$ be a matrix,
    whose singular value decomposition is $\mathbf{A} = \mathbf{U} \bm{\Sigma} \mathbf{V}^\top$, where $\sigma_1 = \|\mathbf{A}\|_2$ is the principal singular value.
    If a small perturbation $\eta \mathbf{B}$ is applied to it,
    then the change in its spectral norm is approximately equal to the projection of the perturbation $\eta \mathbf{B}$ onto the principal singular directions of $\mathbf{A}$, i.e.:
    \begin{equation}\label{eq:perturbation}
        \| \mathbf{A} + \eta \mathbf{B}\|_2 = \sigma_1 + \eta \cdot \langle \mathbf{u}_1^\top \mathbf{B}, \mathbf{v}_1^\top \rangle + o(\eta \|\mathbf{B}\|_2)
    \end{equation}
    where $\mathbf{u}_1, \mathbf{v}_1$ are the principal left and right singular vectors of matrix $\mathbf{A}$, respectively.
\end{theorem}

\begin{proof}
    Let the function be $f(\eta) = \|\mathbf{A}+\eta\mathbf{B}\|_2$.
    We compute the first-order Taylor expansion of $f(\eta)$ at $\eta=0$. According to the definition of the spectral norm, we have $f(\eta)^2 = \lambda_{\max}((\mathbf{A}+\eta\mathbf{B})^\top (\mathbf{A}+\eta\mathbf{B}))$, where $\lambda_{\max}(\cdot)$ denotes the largest eigenvalue of a symmetric matrix.

    Let $\mathbf{M}(\eta)=(\mathbf{A}+\eta\mathbf{B})^\top (\mathbf{A}+\eta\mathbf{B})$. Expanding this gives:
    $$
    \mathbf{M}(\eta) = \mathbf{A}^\top \mathbf{A} + \eta(\mathbf{A}^\top \mathbf{B} + \mathbf{B}^\top \mathbf{A}) + \eta^2 \mathbf{B}^\top \mathbf{B}
    $$
    Let $\lambda(\eta) = \lambda_{\max}(\mathbf{M}(\eta))$ and $\mathbf{v}(\eta)$ be its corresponding normalized eigenvector. In the unperturbed case, i.e., when $\eta=0$, $\mathbf{M}(0)=\mathbf{A}^\top \mathbf{A}$. Its largest eigenvalue is $\lambda(0)=\sigma_1^2$, and the corresponding eigenvector is $\mathbf{v}(0)=\mathbf{v}_1$.

    According to the eigenvalue perturbation theory for symmetric matrices, the derivative of the eigenvalue with respect to the perturbation parameter is given by:
    $$
    \lambda'(\eta) = \mathbf{v}(\eta)^\top \mathbf{M}'(\eta) \mathbf{v}(\eta)
    $$
    where $\mathbf{M}'(\eta) = \frac{d\mathbf{M}(\eta)}{d\eta}$. Differentiating $\mathbf{M}(\eta)$, we get:
    $$
    \mathbf{M}'(\eta) = \mathbf{A}^\top \mathbf{B} + \mathbf{B}^\top \mathbf{A} + 2\eta\mathbf{B}^\top \mathbf{B}
    $$
    At $\eta=0$, the derivative is $\mathbf{M}'(0) = \mathbf{A}^\top \mathbf{B} + \mathbf{B}^\top \mathbf{A}$. Therefore, the derivative of the eigenvalue at $\eta=0$ is:
    $$
    \lambda'(0) = \mathbf{v}(0)^\top \mathbf{M}'(0) \mathbf{v}(0) = \mathbf{v}_1^\top (\mathbf{A}^\top \mathbf{B} + \mathbf{B}^\top \mathbf{A}) \mathbf{v}_1
    $$
    Expanding this expression:
    $$
    \lambda'(0) = \mathbf{v}_1^\top \mathbf{A}^\top \mathbf{B} \mathbf{v}_1 + \mathbf{v}_1^\top \mathbf{B}^\top \mathbf{A} \mathbf{v}_1
    $$
    From the properties of singular value decomposition, we have $\mathbf{A}\mathbf{v}_1 = \sigma_1 \mathbf{u}_1$, and thus $(\mathbf{A}\mathbf{v}_1)^\top = \mathbf{v}_1^\top \mathbf{A}^\top = \sigma_1 \mathbf{u}_1^\top$. Substituting this into the equation gives:
    $$
    \lambda'(0) = \sigma_1 \mathbf{u}_1^\top \mathbf{B} \mathbf{v}_1 + \mathbf{v}_1^\top \mathbf{B}^\top (\sigma_1 \mathbf{u}_1) = \sigma_1 \mathbf{u}_1^\top \mathbf{B} \mathbf{v}_1 + \sigma_1 (\mathbf{u}_1^\top \mathbf{B} \mathbf{v}_1)^\top
    $$
    Since $\mathbf{u}_1^\top \mathbf{B} \mathbf{v}_1$ is a scalar, its value is equal to its transpose, so:
    $$
    \lambda'(0) = 2\sigma_1 \mathbf{u}_1^\top \mathbf{B} \mathbf{v}_1
    $$
    Now, we perform a first-order Taylor expansion of $\lambda(\eta) = \|\mathbf{A}+\eta\mathbf{B}\|_2^2$ at $\eta=0$:
    $$
    \lambda(\eta) = \lambda(0) + \eta\lambda'(0) + O(\eta^2)
    $$
    $$
    \|\mathbf{A}+\eta\mathbf{B}\|_2^2 = \sigma_1^2 + 2\eta\sigma_1 (\mathbf{u}_1^\top \mathbf{B} \mathbf{v}_1) + O(\eta^2)
    $$
    To find $\|\mathbf{A}+\eta\mathbf{B}\|_2$, we take the square root of both sides of the equation:
    \begin{align*}
        \|\mathbf{A}+\eta\mathbf{B}\|_2 &= \sqrt{\sigma_1^2 + 2\eta\sigma_1 (\mathbf{u}_1^\top \mathbf{B} \mathbf{v}_1) + O(\eta^2)} \\
        &= \sigma_1 \sqrt{1 + \frac{2\eta}{\sigma_1}(\mathbf{u}_1^\top \mathbf{B} \mathbf{v}_1) + O(\eta^2)}
    \end{align*}
    Using the approximation $\sqrt{1+x} = 1 + \frac{1}{2}x + O(x^2)$ for $x \to 0$, let $x = \frac{2\eta}{\sigma_1}(\mathbf{u}_1^\top \mathbf{B} \mathbf{v}_1) + O(\eta^2)$, we get:
    \begin{align*}
        \|\mathbf{A}+\eta\mathbf{B}\|_2 &= \sigma_1 \left(1 + \frac{1}{2} \cdot \frac{2\eta}{\sigma_1}(\mathbf{u}_1^\top \mathbf{B} \mathbf{v}_1) + O(\eta^2)\right) \\
        &= \sigma_1 \left(1 + \frac{\eta}{\sigma_1}(\mathbf{u}_1^\top \mathbf{B} \mathbf{v}_1)\right) + o(\eta) \\
        &= \sigma_1 + \eta(\mathbf{u}_1^\top \mathbf{B} \mathbf{v}_1) + o(\eta)
    \end{align*}
    Writing the inner product form $\mathbf{u}_1^\top \mathbf{B} \mathbf{v}_1$ as
    $\eta \cdot \langle \mathbf{u}_1^\top \mathbf{B}, \mathbf{v}_1^\top \rangle$,
    we finally obtain the expression for the spectral norm of $\mathbf{A}+\eta\mathbf{B}$:
    $$
    \|\mathbf{A}+\eta\mathbf{B}\|_2 = \sigma_1 + \eta \cdot \langle \mathbf{u}_1^\top \mathbf{B}, \mathbf{v}_1^\top \rangle + o(\eta\|\mathbf{B}\|_2)
    $$
\end{proof}

Substituting the expression 
$\mathbf{W}_{l, \text{new}} = \mathbf{W}_l - \eta \frac{\partial \mathcal{L}}{\partial \mathbf{W}_l}$
into the above derivation,
we generally consider the parameter update and learning rate 
$\eta$ to be small quantities, 
the change in spectral norm is approximately:
$$
\Delta\|\mathbf{W}_l\|_2 \approx \langle \mathbf{u}_1^\top \left(-\eta \frac{\partial \mathcal{L}}{\partial \mathbf{W}_l}\right), \mathbf{v}_1 \rangle = -\eta \cdot \mathbf{u}_1^\top \frac{\partial \mathcal{L}}{\partial \mathbf{W}_l} \mathbf{v}_1
$$

\subsection{Proof of \cref{thm:spectral_growth}}

To prove \cref{thm:spectral_growth}, 
we first need to derive the change in the spectral norm $\|\mathbf{W}_l\|_2$ 
during a single gradient descent step. 
The exact expression is given by the following theorem:

\begin{theorem}[Change of the Spectral Norm]\label{thm:spectral_growth_expression}
    In a network layer
    that satisfies the spectral alignment condition, 
    the change of the spectral norm $\|\mathbf{W}_l\|_2$ 
    in a single gradient descent step can be approximated as:
    $$
    \Delta\|\mathbf{W}_l\|_2 \approx -\eta\rho\alpha
    \frac{\langle \mathbf{v}_1^\top, \mathbf{f}^{(l)} \rangle}{\|\mathbf{f}^{(l)}\|_2}
    \|\mathbf{h}^{(l-1)}\|_2^2 (\mathbf{p}-\mathbf{t})(\mathbf{h}^{(L)})^\top
    $$
    where $\eta$ is the learning rate, $\rho$ is the gradient propagation factor.
\end{theorem}

\begin{proof}
    By substituting the expression 
    $\mathbf{W}_{l, \text{new}} = \mathbf{W}_l - \eta \frac{\partial \mathcal{L}}{\partial \mathbf{W}_l}$
    into \cref{thm:spectral_norm_change},
    $\Delta\|\mathbf{W}_l\|_2$ is approximately equal to 
    the projection of the gradient descent update term $-\eta \frac{\partial \mathcal{L}}{\partial \mathbf{W}_l}$ onto its principal singular directions, i.e.:
    $$
    \Delta\|\mathbf{W}_l\|_2 \approx \langle \mathbf{u}_1^\top \left(-\eta \frac{\partial \mathcal{L}}{\partial \mathbf{W}_l}\right), \mathbf{v}_1 \rangle = -\eta \cdot \mathbf{u}_1^\top \frac{\partial \mathcal{L}}{\partial \mathbf{W}_l} \mathbf{v}_1
    $$
    where $\eta$ is the learning rate, and $\mathbf{u}_1$ and $\mathbf{v}_1$ are the top left and right singular vectors of $\mathbf{W}_l$, respectively.
    We substitute the approximate gradient expression derived earlier in Theorem \ref{thm:gradient_expression} into the above equation. The gradient expression is:
    $$
    \frac{\partial \mathcal{L}}{\partial \mathbf{W}_l} \approx (\mathbf{h}^{(l-1)})^\top (\mathbf{p}-\mathbf{t})\rho \prod_{k=L}^{l+1} \mathbf{W}_k^\top \cdot \text{Diag}(\mathds{1}(\mathbf{f}^{(l)} > 0))
    $$
    After substituting the gradient expression, the change in the spectral norm is:
    $$
    \Delta\|\mathbf{W}_l\|_2 \approx -\eta \cdot \mathbf{u}_1^\top \left( (\mathbf{h}^{(l-1)})^\top (\mathbf{p}-\mathbf{t})\rho \prod_{k=L}^{l+1} \mathbf{W}_k^\top \cdot \text{Diag}(\mathds{1}(\mathbf{f}^{(l)} > 0)) \right) \mathbf{v}_1
    $$
    Using the associative property of matrix multiplication,
    we rearrange the terms in the above equation and factor out the scalar $(-\eta\rho)$ to get the expression:
    $$
    \Delta\|\mathbf{W}_l\|_2 \approx -\eta \rho \cdot \left( \mathbf{u}_1^\top (\mathbf{h}^{(l-1)})^\top \right) \cdot \left( (\mathbf{p}-\mathbf{t}) \left(\prod_{k=L}^{l+1} \mathbf{W}_k^\top\right) \text{Diag}(\mathds{1}(\mathbf{f}^{(l)} > 0)) \mathbf{v}_1 \right)
    $$
    In the above equation, the first term $\left( \mathbf{u}_1^\top (\mathbf{h}^{(l-1)})^\top \right)$ is the inner product of $\mathbf{u}_1$ and $\mathbf{h}^{(l-1)}$,
    and the second term is the effect of the backpropagated gradient from the output layer on the right singular vector $\mathbf{v}_1$.
    Note that the product of a diagonal matrix and a vector $\text{Diag}(\mathbf{a})\mathbf{v}$ is equivalent to the Hadamard product of the two vectors $\mathbf{a} \odot \mathbf{v}$,
    so we can rewrite the expression as:
    $$
    \Delta\|\mathbf{W}_l\|_2 \approx -\eta\rho \cdot \langle \mathbf{u}_1, (\mathbf{h}^{(l-1)})^\top \rangle \cdot  (\mathbf{p}-\mathbf{t}) \left(\prod_{k=L}^{l+1} \mathbf{W}_k^\top\right) \left(\mathds{1}(\mathbf{f}^{(l)} > 0) \odot \mathbf{v}_1\right) 
    $$
    where $\odot$ denotes the Hadamard product (element-wise product).

    We substitute the orthogonal decomposition from the spectral alignment condition \ref{ass:alignment} into the above equation, using the column vector form of the decomposition, i.e.,
    $\mathbf{u}_1 = \alpha (\mathbf{h}^{(l-1)})^\top + \boldsymbol{\epsilon}^\top$.

    Substituting the decomposition of $\mathbf{u}_1$ into the first term $\langle \mathbf{u}_1, (\mathbf{h}^{(l-1)})^\top \rangle$, yielding:
    $$
    \langle \mathbf{u}_1, (\mathbf{h}^{(l-1)})^\top \rangle = \alpha \langle (\mathbf{h}^{(l-1)})^\top, (\mathbf{h}^{(l-1)})^\top \rangle + \langle \boldsymbol{\epsilon}^\top, (\mathbf{h}^{(l-1)})^\top \rangle
    $$
    Since $\boldsymbol{\epsilon}$ and $\mathbf{h}^{(l-1)}$ are orthogonal in the decomposition, the second term in the above equation is zero.
    Therefore, the inner product term $\langle \mathbf{u}_1, (\mathbf{h}^{(l-1)})^\top \rangle$ simplifies to $\alpha \langle \mathbf{h}^{(l-1)}, \mathbf{h}^{(l-1)} \rangle = \alpha \|\mathbf{h}^{(l-1)}\|_2^2$.

    Now we handle the second term containing $\mathbf{v}_1$, $(\mathbf{p}-\mathbf{t}) \left(\prod_{k=L}^{l+1} \mathbf{W}_k^\top\right) \left(\mathds{1}(\mathbf{f}^{(l)} > 0) \odot \mathbf{v}_1\right)$.
    Note that the condition $\|\boldsymbol{\epsilon}_1\| \ll \|\alpha_1 \mathbf{f}^{(l)}\|$ holds,
    so if we perform an orthogonal decomposition of $h^{(l-1)}$ as $h^{(l-1)} = \beta u_1^\top + \delta$, 
    we similarly have $\|\delta\| \ll \beta$. 
    Substituting into the forward propagation equation yields 
    $f^{(l)} = \beta u_1^\top \mathbf{W}_l + \delta \mathbf{W}_l = \beta \sigma_1 v_1^\top + \delta \mathbf{W}_l$, 
    where the leading singular value $\sigma_1 = \|\mathbf{W}_l\|_2$ 
    is the spectral norm. 
    We have $\|\delta \mathbf{W}_l\|_2 \leq \|\delta\|_2 \cdot \|\mathbf{W}_l\|_2 \ll \beta \|\mathbf{W}_l\|_2 = \| \beta \sigma_1 v_1^\top \|_2$, 
    and note that $\beta \sigma_1 v_1^\top$ 
    is precisely in the direction of the principal right singular vector. 
    Therefore, if we perform an orthogonal decomposition of $f^{(l)}$ with respect to $v_1$, 
    the residual orthogonal term remains much smaller than the projection term.
    
    The above discussion can be stated as: 
    when the input $\mathbf{h}$ is aligned with $\mathbf{u}_1$, 
    the pre-activation $\mathbf{f}$ is also aligned with $\mathbf{v}_1$.
    So if we decompose $\mathbf{v}_1$ in the direction of $\mathbf{f}^{(l)}$,
    and substituting the decomposition into the Hadamard product 
    $\mathds{1}(\mathbf{f}^{(l)} > 0) \odot \mathbf{v}_1$.
    Then by previous discussion,
    we can neglect the small residual term,
    and we obtain the expression for the change in spectral norm:
    $$
    \Delta\|\mathbf{W}_l\|_2 \approx 
    -\eta\rho(\alpha \|\mathbf{h}^{(l-1)}\|_2^2) \cdot (\mathbf{p}-\mathbf{t}) \left(\prod_{k=L}^{l+1} \mathbf{W}_k^\top \right) 
    \left(\frac{\langle \mathbf{v}_1^\top, \mathbf{f}^{(l)} \rangle}{\|\mathbf{f}^{(l)}\|_2} (\mathds{1}(\mathbf{f}^{(l)} > 0) \odot (\mathbf{f}^{(l)})^\top)\right)
    $$
    Here $\frac{\langle \mathbf{v}_1^\top, \mathbf{f}^{(l)} \rangle}{\|\mathbf{f}^{(l)}\|_2}$ 
    is the norm of the projection.
    The complex term $ (\mathbf{p}-\mathbf{t}) \left(\prod_{k=L}^{l+1} \mathbf{W}_k^\top \right) 
    \left(\mathds{1}(\mathbf{f}^{(l)} > 0) \odot (\mathbf{f}^{(l)})^\top\right)$ 
    is essentially a forward propagation process, 
    where the linear layers act on the signal originating from the pre-activation $\mathbf{f}^{(l)}$ 
    of the $l$-th layer and map it to the final layer. 
    Therefore, this term can be directly expressed by the final output term $(\mathbf{p}-\mathbf{t})(\mathbf{h}^{(L)})^\top$.

    Combining all the scalars, we get the final expression:
    $$
    \Delta\|\mathbf{W}_l\|_2 \approx -\eta\rho\alpha
    \frac{\langle \mathbf{v}_1^\top, \mathbf{f}^{(l)} \rangle}{\|\mathbf{f}^{(l)}\|_2}
    \|\mathbf{h}^{(l-1)}\|_2^2 (\mathbf{p}-\mathbf{t})(\mathbf{h}^{(L)})^\top
    $$
\end{proof}

To determine the sign of the key term 
$((\mathbf{p}-\mathbf{t})(\mathbf{h}^{(L)})^\top)$ 
in the expression for $\Delta\|\mathbf{W}_l\|_2$, 
we need the following technical lemma, 
which provides the final foundation for the proof of \cref{thm:spectral_growth}.

\begin{lemma}[Sign of the Expected Logit Deviation Term]\label{lem:logit_deviation_sign}
    Let $\mathcal{L}(\theta)$ be the cross-entropy loss function defined by a neural network with parameters $\theta$, which produces a logit vector $\mathbf{z}$ at the output layer.
    The optimization process of minimizing $\mathcal{L}(\theta)$ via gradient descent will necessarily cause the system state to evolve into a region satisfying the inequality $z_k > E_{i \sim p(\mathbf{z})}[z_i]$, where $k$ is the index of the target class.
    Therefore, the sign of the expected logit deviation term $E_{i \sim p(\mathbf{z})}[z_i] - z_k = (\mathbf{p}-\mathbf{t})\mathbf{h}^{(L)}$ is negative.
\end{lemma}

\begin{proof}
    We first analyze the gradient properties of the cross-entropy loss function $\mathcal{L}(\mathbf{z}) = \log(\sum_j e^{z_j}) - z_k$ within the logit space $\mathbb{R}^{n_L}$.
    Since $\mathcal{L}$ is the cross-entropy loss function, its gradient with respect to the final logit output is
    $\nabla_z \mathcal{L} =  \frac{\partial \mathcal{L}}{\partial \mathbf{h}^{(L)}} = \mathbf{p} - \mathbf{t} $,
    where $\mathbf{t}$ is the one-hot vector for the target label $k$,
    and $\mathbf{p}_i = e^{z_i} / \sum_j e^{z_j}$ is the probability obtained via the Softmax function.

    A gradient-based optimization algorithm must converge to a stable point where the gradient of the loss function is zero,
    i.e., the gradient vector $\nabla_{\mathbf{z}} \mathcal{L} = 0$.
    A necessary and sufficient condition for this is that all components are simultaneously zero, i.e., $p_k - 1 = 0$ and $p_i = 0$ for all $i \neq k$.
    Thus, the cross-entropy loss function has a unique stable point in the logit space,
    which corresponds to the state where the model makes a perfect classification (i.e., $p_k = 1$).

    Next, we prove that the optimization process cannot converge to any point that does not satisfy the inequality $z_k > E_{i \sim p(\mathbf{z})}[z_i]$.
    Assume the optimization process converges to a point $\mathbf{z}^*$ that does not satisfy this inequality,
    i.e., $z_k^* \le E_{i \sim p(\mathbf{z}^*)}[z_i]$. We will show that this assumption leads to a logical contradiction.

    The inequality $z_k^* \le E_{i \sim p(\mathbf{z}^*)}[z_i]$ implies $p_k^* < 1$,
    because only when $p_k^* = 1$ (at which point all $p_{i \neq k}^* = 0$) does the expectation take its boundary value $E_{i \sim p(\mathbf{z}^*)}[z_i] = \sum_j p_j^* z_j^* = 1 \cdot z_k^* + \sum_{j \neq k} 0 \cdot z_j^* = z_k^*$.
    In any other case, as long as there is at least one $p_{i \neq k}^* > 0$, the expected value is a weighted average of all logits,
    and the inequality $z_k^* > E_{i \sim p(\mathbf{z}^*)}[z_i]$ would hold.

    Furthermore, according to the gradient expression, at any point where $p_k < 1$, the gradient component $\frac{\partial \mathcal{L}}{\partial z_k} = p_k - 1$ must be negative, which means the gradient vector $\nabla_{\mathbf{z}} \mathcal{L}(\mathbf{z}^*)$ must be non-zero.

    The above assumption (that the optimization process converges to a point $\mathbf{z}^*$ and $z_k^* \le E_{i \sim p(\mathbf{z}^*)}[z_i]$) leads to a contradiction. By the definition of a convergence point, $\mathbf{z}^*$ must be a stable point, i.e., the gradient $\nabla_{\mathbf{z}} \mathcal{L}(\mathbf{z}^*) = 0$.
    However, the condition $z_k^* \le E_{i \sim p(\mathbf{z}^*)}[z_i]$ implies $p_k^* < 1$.
    For any state with $p_k < 1$, the gradient component $\frac{\partial \mathcal{L}}{\partial z_k} = p_k - 1 \neq 0$, hence $\nabla_{\mathbf{z}} \mathcal{L}(\mathbf{z}^*) \neq 0$. Thus, $\mathbf{z}^*$ cannot be a stable point, and the assumption is false.

    From the preceding discussion, the set of convergence points for the optimization algorithm is a subset of the state space region satisfying the inequality $z_k > E_{i \sim p(\mathbf{z})}[z_i]$. Define the state space $S_{>} = \{\mathbf{z} \in \mathbb{R}^{n_L} \mid z_k > E_{i \sim p(\mathbf{z})}[z_i]\}$.
    For any point $\mathbf{z}' \notin S_{>}$ (except for the unique stable point $p_k = 1$), its gradient $\nabla_{\mathbf{z}} \mathcal{L}(\mathbf{z}')$ is non-zero. The gradient descent update vector $\Delta \mathbf{z} = -\eta \nabla_{\mathbf{z}} \mathcal{L}$ ensures an increase in $z_k$ and a decrease in $z_{i \neq k}$, and this update vector field must point towards the region $S_{>}$.
    
    Given that the loss function $\mathcal{L}$ is bounded below and gradient descent ensures $\mathcal{L}(\mathbf{z}_{t+1}) < \mathcal{L}(\mathbf{z}_t)$, the convergence trajectory $\{\mathbf{z}_t\}$ must approach the unique stable point. This convergence process requires the trajectory to leave the non-stable regions, enter, and ultimately remain within the region $S_{>}$.
    Therefore, the sign of the expected logit deviation term $E_{i \sim p(\mathbf{z})}[z_i] - z_k = (\mathbf{p}-\mathbf{t})\mathbf{h}^{(L)}$ is negative.
\end{proof}

Now we can complete the proof of the main \cref{thm:spectral_growth}.
It demonstrates that under the spectral alignment condition, 
the spectral norm exhibits deterministic growth.

\begin{proof}[proof of \cref{thm:spectral_growth}]
    Recall the expression for the change in spectral norm is given by
    $$
    \Delta\|\mathbf{W}_l\|_2 \approx -\eta\rho\alpha
    \frac{\langle \mathbf{v}_1^\top, \mathbf{f}^{(l)} \rangle}{\|\mathbf{f}^{(l)}\|_2}
    \|\mathbf{h}^{(l-1)}\|_2^2 (\mathbf{p}-\mathbf{t})(\mathbf{h}^{(L)})^\top
    $$
    We need to determine the sign of the right-hand side of the above equation.
    In this expression, 
    the learning rate $\eta>0$, 
    the scalar approximation factor $\rho>0$, 
    and the squared norm $\|\mathbf{h}^{(l-1)}\|_2^2 > 0$.
    According to the spectral alignment condition, 
    the alignment strengths $\alpha > 0$. 
    Therefore, the sign of $\Delta\|\mathbf{W}_l\|_2$ is entirely determined by the sign of $-(\mathbf{p}-\mathbf{t})(\mathbf{h}^{(L)})^\top$.
    We expand $(\mathbf{p}-\mathbf{t})(\mathbf{h}^{(L)})^\top$, considering that the logits of the final layer are obtained from the previous layer's activation $\mathbf{h}^{(L-1)}$ and the weight matrix $\mathbf{W}_L$, i.e., $\mathbf{z}=\mathbf{h}^{(L)} = \mathbf{h}^{(L-1)} \mathbf{W}_L$.
    We can write this term as the matrix weights acting on the forward propagation: $(\mathbf{p}-\mathbf{t})\mathbf{W}_L^\top (\mathbf{h}^{(L-1)})^\top$.
    Expanding it, we get:
    \begin{align*}
    (\mathbf{p}-\mathbf{t})\mathbf{W}_L^\top (\mathbf{h}^{(L-1)})^\top &= \left( \sum_{i=1}^{n_L} p_i (\mathbf{W}_L^\top)_{i,:} - (\mathbf{W}_L^\top)_{k,:} \right) (\mathbf{h}^{(L-1)})^\top \\
    &= \sum_{i=1}^{n_L} p_i (\mathbf{W}_L^\top)_{i,:} (\mathbf{h}^{(L-1)})^\top - (\mathbf{W}_L^\top)_{k,:} (\mathbf{h}^{(L-1)})^\top
    \end{align*}
    Here $k$ is the index of the target label. Noting that logits $z_i = (\mathbf{W}_L^\top)_{i,:} (\mathbf{h}^{(L-1)})^\top$, and substituting the definition of Softmax probability $p_i = \frac{e^{z_i}}{\sum_{j=1}^{n_L} e^{z_j}}$, we get the expression for $(\mathbf{p}-\mathbf{t})(\mathbf{h}^{(L)})^\top$:
    $$
    (\mathbf{p}-\mathbf{t})(\mathbf{h}^{(L)})^\top = \sum_{i=1}^{n_L} \left( ((\mathbf{W}_L^\top)_{i,:} (\mathbf{h}^{(L-1)})^\top) \cdot \frac{e^{(\mathbf{W}_L^\top)_{i,:} (\mathbf{h}^{(L-1)})^\top}}{\sum_{j=1}^{n_L} e^{(\mathbf{W}_L^\top)_{j,:} (\mathbf{h}^{(L-1)})^\top}} \right) - (\mathbf{W}_L^\top)_{k,:} (\mathbf{h}^{(L-1)})^\top
    $$
    Substituting the logit definition $z_i$ back into the above equation, the expression simplifies to:
    \begin{align*}
    (\mathbf{p}-\mathbf{t})(\mathbf{h}^{(L)})^\top &= \sum_{i=1}^{n_L} z_i \frac{e^{z_i}}{\sum_{j=1}^{n_L} e^{z_j}} - z_k \\
    &= E_{i \sim p(\mathbf{z})}[z_i] - z_k
    \end{align*}
    By Lemma \ref{lem:logit_deviation_sign}, the gradient properties of the cross-entropy loss function determine its optimization dynamics. In the logit space, any gradient descent step will deterministically increase the value of the target logit $z_k$, and the system state evolves to eventually satisfy the following inequality:
    $$
    z_k > E_{i \sim p(\mathbf{z})}[z_i]
    $$
    From this, we can deduce:
    $$
    (\mathbf{p}-\mathbf{t})(\mathbf{h}^{(L)})^\top = E_{i \sim p(\mathbf{z})}[z_i] - z_k < 0
    $$
    Since the value of $(\mathbf{p}-\mathbf{t})(\mathbf{h}^{(L)})^\top$ is negative, we can finally determine the sign of the change in the spectral norm:
    $$
    \Delta\|\mathbf{W}_l\|_2 \approx - 
    \underbrace{(\eta\rho\alpha
    \frac{\langle \mathbf{v}_1^\top, \mathbf{f}^{(l)} \rangle}{\|\mathbf{f}^{(l)}\|_2}\|\mathbf{h}^{(l-1)}\|_2^2)}_{\text{positive}} \cdot \underbrace{((\mathbf{p}-\mathbf{t})(\mathbf{h}^{(L)})^\top)}_{\text{negative}} > 0
    $$
    This proves that, driven by the cross-entropy loss and satisfying the spectral alignment condition, the spectral norm of the weight matrices has an intrinsic tendency to grow, leading to training collapse.    
\end{proof}

Finally, 
we provide the proof of Corollary~\ref{cor:growth_amplification}, 
which demonstrates that when spectral alignment causes the spectral norm to increase, 
it leads to trigger a growth of activations and gradients. 
This confirms the effectiveness of spectral alignment 
as an early-warning indicator.

\begin{proof}[proof of \cref{cor:growth_amplification}]
    Consider the forward propagation $f^{(l)} = h^{(l-1)} W_l$.
    When training enters the pathological spectral alignment state, 
    we perform an orthogonal decomposition of the input vector $h^{(l-1)}$ 
    with respect to the principal left singular vector $u_1$ of the 
    weight matrix 
    $W_l$: $h^{(l-1)} = \beta u_1^\top + \delta$, 
    where $\beta$ is the projection coefficient and $\delta$ 
    is the residual term orthogonal to $u_1$. 
    Under the alignment condition, $\|\delta\|_2 \ll \beta$.

    Substituting this decomposition into the forward propagation equation, 
    we obtain: 
    $f^{(l)} = (\beta u_1^\top + \delta) W_l = \beta u_1^\top W_l + \delta W_l$. 
    By the properties of singular value decomposition, 
    $u_1^\top W_l = \sigma_1 v_1^\top$, 
    where $\sigma_1 = \|W_l\|_2$ is the spectral norm of the weight matrix 
    and $v_1$ is the principal right singular vector. Therefore,
    $$
    f^{(l)} = \beta \sigma_1 v_1^\top + \delta W_l
    $$
    Since $\|\delta\| \ll \beta$, 
    we have 
    $\|\delta W_l\|_2 \leq \|\delta\|_2 \|W_l\|_2 \ll \beta \|W_l\|_2 = \| \beta \sigma_1 v_1^\top \|_2$, 
    which means the norm of the residual term $\delta W_l$ 
    is much smaller than that of the projection term 
    $\beta \sigma_1 v_1^\top$. 
    Thus, the norm of the pre-activation $f^{(l)}$ 
    can be approximated as $\|f^{(l)}\|_2 \approx \beta \|W_l\|_2$.

    \cref{thm:spectral_growth} 
    proves that under the pathological alignment condition, 
    the spectral norm $\|W_l\|_2$ increases in a single gradient descent step. 
    According to the above discussion, when $\|W_l\|_2$ increases, 
    the norm of the pre-activation $f^{(l)}$ also increases. 
    Since activation functions such as ReLU are monotonic, 
    this will cause the norm of the activation 
    $h^{(l)}$ to increase as well. 

    Now consider the gradient in the backward propagation process. 
    By \cref{thm:gradient_expression}, 
    the approximate expression for the gradient is
    $$
    \frac{\partial L}{\partial W_l} \approx (h^{(l-1)})^\top (p-t) \rho \left( \prod_{k=L}^{l+1} W_k^\top \right) \cdot \mathrm{Diag}(1(f^{(l)} > 0))
    $$
    where
    $\left(\prod_{k=L}^{l+1} W_k^\top\right)$ 
    is the product of transposed weight matrices in the backward pass. 
    As analyzed in \cref{thm:spectral_growth_expression} for forward propagation, 
    when the input $h^{(l-1)}$ is aligned with $u_1$, 
    the corresponding pre-activation $f^{(l)}$ is aligned with $v_1$, 
    causing the forward signal to be amplified along 
    the principal singular directions of each weight matrix.

    In the gradient expression, 
    the key term 
    $\left( \prod_{k=L}^{l+1} W_k^\top \right) \cdot \mathrm{Diag}(1(f^{(l)} > 0))$ 
    essentially represents the network's forward propagation. 
    Since the forward signal is forcibly amplified along the principal singular directions, 
    the backward gradient signal will also propagate along these directions. 
    Under the pathological alignment condition, 
    the spectral norm $\|W_k\|_2$ of each weight matrix grows, 
    causing the backward gradient signal to be repeatedly amplified.

    Therefore, 
    the growth trend of the backward gradient norm can be approximated 
    by the product of the spectral norms of each weight matrix:
    $$
    \left\|\frac{\partial L}{\partial W_l}\right\|_F \propto \|h^{(l-1)}\|_2 \cdot \prod_{k=L}^{l+1} \|W_k\|_2
    $$
    When the spectral norm $\|W_k\|_2$ increases, 
    the gradient norm will also increase,
    leading to loss explosion.
\end{proof}

\end{document}